\documentclass[11pt]{article}
\usepackage{geometry}                % See geometry.pdf to learn the layout options. There are lots.
\usepackage{setspace}
\usepackage[parfill]{parskip}    % Activate to begin paragraphs with an empty line rather than an indent

\usepackage{graphicx}
\usepackage{amssymb}

\usepackage{algorithm, algorithmic}
\usepackage{amsmath}
\usepackage{amsthm}
\newtheorem{theorem}{Theorem}
\newtheorem{definition}{Definition}
\newtheorem{corollary}{Corollary}
\newtheorem{remark}{Remark}
\newtheorem{lemma}{Lemma}
\newtheorem{assumption}{Assumption}

\usepackage{mathtools}
\DeclarePairedDelimiter{\ceil}{\lceil}{\rceil}

\usepackage[hidelinks]{hyperref}

\usepackage{color}

\usepackage{booktabs}

\usepackage{bm}

\usepackage{newtxtext}

\title{Tighter Generalization Bounds for\\Iterative Differentially Private Learning Algorithms}

\author{Fengxiang He~\thanks{Both authors contributed equally.}~~\thanks{The authors were with UBTECH Sydney AI Centre, School of Computer Science, Faculty of Engineering, the University of Sydney, Darlington NSW 2008, Australia. Email: \href{mailto:fengxiang.he@sydney.edu.au}{fengxiang.he@sydney.edu.au}, \href{mailto:bhwangfy@gmail.com}{bhwangfy@gmail.com}, and \href{mailto:dacheng.tao@sydney.edu.au}{dacheng.tao@sydney.edu.au}.} \and Bohan Wang~\footnotemark[1]~~\footnotemark[2]~~\thanks{Bohan Wang was also with University of Science and Technology of China. This work was completed when he was an intern at the University of Sydney.} \and Dacheng Tao~\footnotemark[2]}

\date{}                                           % Activate to display a given date or no date

\begin{document}

\maketitle

\begin{abstract}
This paper studies the relationship between generalization and privacy preservation in iterative learning algorithms by two sequential steps. We first establish an alignment between generalization and privacy preservation for any learning algorithm. We prove that $(\varepsilon, \delta)$-differential privacy implies an on-average generalization bound for multi-database learning algorithms which further leads to a high-probability bound for any learning algorithm. This high-probability bound also implies a PAC-learnable guarantee for differentially private learning algorithms. We then investigate how the iterative nature shared by most learning algorithms influence privacy preservation and further generalization. Three composition theorems are proposed to approximate the differential privacy of any iterative algorithm through the differential privacy of its every iteration. By integrating the above two steps, we eventually deliver generalization bounds for iterative learning algorithms, which suggest one can simultaneously enhance privacy preservation and generalization. Our results are strictly tighter than the existing works. Particularly, our generalization bounds do not rely on the model size which is prohibitively large in deep learning. This sheds light to understanding the generalizability of deep learning. These results apply to a wide spectrum of learning algorithms. In this paper, we apply them to stochastic gradient Langevin dynamics and agnostic federated learning as examples.
\end{abstract}

%\bigskip

~~~~~~~~~~\textbf{Keywords:} learning theory, differential privacy, generalization.

\addcontentsline{toc}{section}{Abstract}

\maketitle

\section{Introduction}

Generalization to unseen data and privacy preservation are two increasingly important facets of machine learning. Specifically, good generalization guarantees that an algorithm learns the underlying patterns in the training data rather than just memorizes the data \cite{vapnik2013nature, mohri2018foundations}. In this way, good generalization abilities provide confidence that the models trained on existing data can be applied to similar but unseen scenarios. Additionally, massive personal data has been collected, such as financial and medical records. How to discover the highly valuable population knowledge carried in the data while protecting the highly sensitive individual privacy has profound importance \cite{dwork2014algorithmic, pittaluga2016pre}. 

%\subsection{Our Results}

This paper investigates the relationship between generalization and privacy preservation in iterative machine learning algorithms by the following two steps: (1) exploring the relationship between generalization and privacy preservation in any learning algorithm; and (2) analyzing how the iterative nature shared by most learning algorithms would influence the privacy-preserving ability and further the generalizability.

We first prove two theorems that upper bound the generalization error of an learning algorithm via its differential privacy. %, ignoring the iterative nature shared by most learning algorithms. 
Specifically, we prove a high-probability upper bound for the generalization error,
\begin{equation*}
\mathcal R(\mathcal{A}(S)) - \hat{\mathcal R}_S (\mathcal{A}(S)),
\end{equation*}
where $\mathcal{A}(S)$ is the hypothesis learned by algorithm $\mathcal{A}$ on the training sample set $S$, $\mathcal R(\mathcal{A}(S))$ is the expected risk, and $\hat{\mathcal R}_S (\mathcal{A}(S))$ is the empirical risk. This bound is established based on a novel on-average generalization bound for any $(\varepsilon,\delta)$-differentially private multi-database learning algorithm. Our high-probability generalization bound further implies that differentially private machine learning algorithms are probably approximately correct (PAC)-learnable. These results indicate that the algorithms with a good privacy-preserving ability also have a good generalizability. We, therefore, can expect to design novel learning algorithms for better generalizability by enhancing its privacy-preserving ability.

We then studied how the iterative nature shared by most learning algorithms influences the privacy-preserving ability and further the generalizability. %, %which is a common property shared by most machine learning algorithms.
Generally, the privacy-preserving ability of an iterative algorithm degenerates along with iterations, since the amount of leaked information cumulates when the algorithm is progressing. To capture this degenerative property, we further prove three composition theorems that calculate the differential privacy of any iterative algorithm via the differential privacy of its every iteration. Combining with the established relationship between generalization and privacy preservation, our composition theorems help characterize the generalizabilities of iterative learning algorithms.

Our results considerably extend the current understanding of the relationship between generalization and privacy preservation in iterative learning algorithms.

Existing works \cite{dwork2015preserving,  nissim2015generalization, oneto2017differential} have proved some high-probability generalization bounds in the following form,
\begin{equation*}
\mathbb P\left[\mathcal R(\mathcal{A}(S)) - \hat{\mathcal R}_S (\mathcal{A}(S)) > a \right] < b,
\end{equation*}
where $a$ and $b$ are two positive constant real numbers. Our high-probability bound is strictly tighter than the current tightest results by Nissim and Stemmer \cite{nissim2015generalization} from two aspects: (1) our bound tightens the term $a$ from $13\varepsilon$ to $9\varepsilon$; and (2) our bound tightens the term $b$ from $\frac{2 \delta}{\varepsilon} \log \left(\frac{2}{\varepsilon}\right)$ to $\frac{2 e^{-\varepsilon}\delta}{\varepsilon} \log \left(\frac{2}{\varepsilon}\right)$. Besides, we prove a PAC-learnable guarantee for differentially private machine learning algorithms via the high-probability generalization bound. Nissim and Stemmer also proved an on-average multi-database generalization bound. Our on-average multi-database generalization bound is tighter by a factor of $e^\varepsilon$. These improvements are significant in practice because the factor $\varepsilon$ can be as large as $10$ in the experiments by Adabi et al. \cite{abadi2016deep}. Also, the bounds by Nissim and Stemmer \cite{nissim2015generalization} are only for binary classification, while ours apply to any differentially private learning algorithm.

Some works have also proved composition theorems \cite{dwork2014algorithmic, kairouz2017composition}. The approximation of factor $\delta$ in our composition theorems is tighter than the tightest existing result \cite{kairouz2017composition} by
\begin{equation*}
\delta \frac{e^\varepsilon-1}{e^\varepsilon+1} \left(T-\left\lceil \frac{\varepsilon'}{\varepsilon}\right\rceil\right),
\end{equation*}
where $T$ is number of iterations, while the estimate of $\varepsilon'$ remains the same. This improvement is significant because the iteration number $T$ can be considerably large in practice. This helps our composition theorems to further tighten our generalization bounds for iterative learning algorithms considerably.

Our results apply to a wide spectrum of machine learning algorithms. This paper applies them to stochastic gradient Langevin dynamics \cite{welling2011bayesian} as an example of the stochastic gradient Markov chain Monte Carlo scheme \cite{ma2015complete} and agnostic federated learning \cite{geyer2017differentially}. 
%(see Section \ref{sec:applications}, Theorems \ref{thm:SGLD}). %; and Section \ref{sec:federated_learning}, Theorems \ref{thm:federated}, respectively). 
Our results deliver %$\mathcal O(T\sqrt{\log N}/N)$ on-average 
generalization bounds %and %$\mathcal O(T\sqrt{\log N/N})$  high-probability generalization bounds 
for SGLD and agnostic federated learning. %, where $N$ is the training sample size. 
The obtained generalization bounds do not explicitly rely on the model size, which can be prohibitively large in modern methods, such as deep neural networks. %These implementations are natural but technically non-trivial.

The rest of this paper is organized as follows. Section \ref{sec:review} reviews related works  of generalization, differential privacy, and related issues in deep learning theory. Section \ref{sec:preliminary} defines notations and recalls preliminaries. Section \ref{sec:main_result} provides main results: Section \ref{sec:generalization_privacy} establishes the relationship between generalization and privacy preservation, sketches the proofs wherein, and compares existing results with ours; and Section \ref{sec:generalization_iterative_private} establishes the degenerative nature of differential privacy in iterative algorithms and its influence to the generalizability, sketches the proofs wherein, and compares existing results with ours. Section \ref{sec:applications} applies our results to two popular schemes: stochastic gradient Langevin dynamics (Section \ref{sec:SGLD}) and agnostic federated learning (Section \ref{sec:federated_learning}). Section \ref{sec:conclusion} concludes this paper.

\section{Background}
\label{sec:review}

Generalization bound is a standard measurement of the generalizability, which is defined as the upper bound of the difference between the expected risk and the empirical risk \cite{musavi1994generalization, vapnik2013nature, mohri2018foundations}. Since the two risks can be treated as the training error and the expectation of the test error, generalization bound expresses the  gap between the performance on existing data and the performance on unknown data.  Therefore, we can expect an algorithm with a small generalization bound to generalize well. Existing generalization bounds are mainly obtained from three stems: (1) concentration inequalities derive many high-probability generalization bounds based on the hypothesis complexity, such as VC dimension \cite{blumer1989learnability, vapnik2006estimation}, Rademacher complexity \cite{koltchinskii2000rademacher, koltchinskii2001rademacher, bartlett2002rademacher}, and covering number \cite{dudley1967sizes, haussler1995sphere}. These generalization bounds suggest implementations  consistent with the principle of Occam's razor that controls the hypothesis complexity to help models generalize better; (2) some on-average and high-probability generalization bounds are proved based on the algorithmic stability to the disturbance in the training sample set \cite{rogers1978finite, bousquet2002stability, xu2011sparse}, following an intuition that an algorithm with good generalization ability is insensitive to the disturbance in individual data points; and (3) under the PAC-Bayes framework \cite{mcallester1999pac, mcallester1999some}, generalization bounds are established on information-theoretical distances between the output hypothesis and the prior, such as KL divergence and mutual information.

Differential privacy measures an algorithm according to its privacy-preserving ability \cite{dwork2014algorithmic, 10.1007/11787006_1}. Specifically, $(\varepsilon, \delta)$-differential privacy is defined as the change in output hypothesis when the algorithm $\mathcal{A}$ is exposed to attacks as follows,
\begin{equation*}
	%\label{eq:dp}
	\log \left[ \frac{\mathbb P_{\mathcal{A}(S)}(\mathcal{A}(S)\in B) - \delta}{\mathbb P_{\mathcal{A}(S')}(\mathcal{A}(S')\in B)} \right] \le \varepsilon,
\end{equation*}
where $B$ is an arbitrary subset of the hypothesis space and $(S, S')$ is a neighboring sample set pair, in which $S$ and $S'$ only differ by one example. Therefore, an algorithm with small differential privacy $(\varepsilon, \delta)$ robust to changes in individual training examples. Thus, the magnitude of differential privacy $(\varepsilon, \delta)$ indexes the ability to resist  {\it differential attacks} that uses fake sample points as probes to attack machine learning algorithms, and then infer the individual privacy via the changes of output hypotheses. Many variants of differential privacy have been designed by modifying the division operation: (1) concentration differential privacy assumes that the privacy loss (cf. \cite{dwork2014algorithmic}, p.18) defined as below,
\begin{equation*}
\log \left[ \frac{\mathbb P_{\mathcal{A}(S)}(\mathcal{A}(S)\in B)}{\mathbb P_{\mathcal{A}(S')}(\mathcal{A}(S')\in B)} \right],
\end{equation*}
is sub-Gaussian \cite{dwork2016concentrated, bun2016concentrated}; (2) mutual-information differential privacy and KL differential privacy adapt mutual information and KL divergence, respectively, to measure changes of the hypotheses \cite{cuff2016differential, wang2016relation, liao2017hypothesis, chaudhuri2019capacity}; (3) R\'enyi differential privacy further replaces the KL divergence by R\'enyi divergence \cite{mironov2017renyi, geumlek2017renyi}; etc.

As an over-parameterized model, deep learning has demonstrated excellent generalizability, which is somehow beyond the explanation of the existing statistical learning theory and thus attracts the community's interest. Recent advances include generalization bounds via VC dimension \cite{JMLR:v20:17-612}, Rademacher complexity \cite{golowich2017size, bartlett2017spectrally}, covering number \cite{bartlett2017spectrally}, Fisher-Rao norm \cite{liang2019fisher, tu2020understanding}, PAC-Bayesian framework \cite{neyshabur2017pac}, algorithmic stability \cite{hardt2016train}, and the dynamics of stochastic gradient descent or its variant \cite{mandt2017stochastic, mou2018generalization, he2019control} driven by the loss surface \cite{kawaguchi2016deep, yun2019small, he2020piecewise}. A major difficulty in explaining deep learning's excellent generalizability is that deep learning models usually has prohibitively large parameter size which make many generalization bounds vacuous.
Additionally, as a game-changer, deep learning has become a dominant player in many real-world application areas, including financial services \cite{fischer2018deep}, healthcare \cite{wang2012framework}, and biometric authentication \cite{snelick2005large}, in which the privacy-preserving ability is of vital importance. Several works have also studied the privacy preservation of deep learning and how to improve it further \cite{abadi2016deep, arachchige2019local}. This work establishes generalization bounds for iterative learning algorithms via differential privacy, which do not explicitly rely on the model size. Our results also shed light to fully understanding the generalizability of deep learning from the privacy-preserving view.

\section{Notations and Preliminaries}
\label{sec:preliminary}

Suppose $S = \{(x_1, y_1), \ldots, (x_N, y_N) | x_i \in \mathcal X \subset \mathbb R^{d_X}, y_i \in \mathcal Y \subset \mathbb R^{d_Y}, i = 1, \ldots, N\}$ is a training sample set, where $x_i$ is the $i$-th feature, $y_i$ is the corresponding label, and $d_X$ and $d_Y$ are the dimensions of the feature and the label, respectively. For the brevity, we define $z_i = (x_i, y_i)$. We also define random variables $Z = (X, Y)$, such that all $z_i = (x_i, y_i)$ are independent and identically distributed (i.i.d.) observations of the variable $Z = (X, Y) \in \mathcal Z,~ Z \sim \mathcal D$, where $\mathcal D$ is the data distribution.

A machine learning algorithm $\mathcal A$ learns a hypothesis,
\begin{equation*}
\mathcal A(S) \in \mathcal H \subset \mathcal Y^{\mathcal X} = \{f: \mathcal X \to \mathcal Y\},
\end{equation*}
from the training sample $S \in \mathcal Z^N$.
%Generalization bound measures the generalization ability of an algorithm. 
%For the hypothesis $\mathcal A(S)$ learned by an algorithm $\mathcal A$ on the training sample set $S$, %the expected risk $\mathcal R(h)$ and empirical risk $\hat{\mathcal R}_S(h)$ with respect to the training sample set $S$ are respectively defined as follows,
%\begin{gather*}
%\label{eq:expected_risk}
%	\mathcal R(h) = \mathbb E_Z l(h, Z),~
%\label{eq:empirical_risk}
%	\hat{\mathcal R}_S(h) = \frac{1}{N} \sum_{i=1}^N l(h, z_i).
%\end{gather*}
%Additionally, when the output hypothesis $h$ of the machine learning algorithm $\mathcal{A}$ is stochastic, we usually calculate the expectations of 
The expected risk $\mathcal R(\mathcal A(S))$ and empirical risk $\hat{\mathcal R}(\mathcal A(S))$ of the algorithm $\mathcal A$ are defined as follows,
\begin{gather*}
	 \mathcal R(\mathcal A(S)) = \mathbb E_{Z} l(A(S), Z),\\
	\hat{\mathcal R}_S(\mathcal A(S)) = \frac{1}{N} \sum_{i=1}^N l(A(S), z_i),
\end{gather*}
where $l: \mathcal H \times \mathcal Z \to \mathbb R^+$ is the loss function. It is worth noting that both the algorithm $\mathcal A$ and the training sample set $S$ can introduce randomness in the expected risk $\mathcal R(\mathcal A(S))$ and empirical risk $\hat{\mathcal R}(\mathcal A(S))$.
%When the algorithm is stochastic, we can also calculate "expectations" of expected risk and empirical risk:
%\begin{gather*}
%	\mathcal R(\mathcal{A}) = \mathbb E_{S, Z} l(A(S), Z),~
%	\hat{\mathcal R}_S(\mathcal{A}) = \frac{1}{N} \sum_{i=1}^N \mathbb{E}_{S} l(A(S), z_i).
%\end{gather*}
The generalization error is defined as the difference between the expected risk and empirical risk, 
\begin{equation*}
\hat{\mathcal{R}}_S(\mathcal{A}(S)) - \mathcal{R}(\mathcal{A}(S)),
\end{equation*}
whose upper bound is called the generalization bound.

%Privacy-preserving ability is important to machine learning. 
Differential privacy measures the ability to preserve privacy, which is defined as follows.

\begin{definition}[Differential Privacy; cf. \cite{dwork2014algorithmic}]%, pp. 17-18]
		\label{def:dp}
		A stochastic algorithm $\mathcal{A}$ is called ($\varepsilon,\delta$)-differentially private if for any subset $B \subset \mathcal H$ and any neighboring sample set pair $S$ and $S'$ which are different by only one example, we have
		\begin{equation*}
		\label{eq:dp}
		\log \left[ \frac{\mathbb P_{\mathcal{A}(S)}(\mathcal{A}(S)\in B) - \delta}{\mathbb P_{\mathcal{A}(S')}(\mathcal{A}(S')\in B)} \right] \le \varepsilon.
		\end{equation*}
		The algorithm $\mathcal{A}$ is also called $\varepsilon$-differentially private, if it is $(\varepsilon, 0)$-differentially private.
\end{definition}

Differential privacy measures the distance between the hypotheses learned from neighboring training sample sets which are different by only one single example. Some (pseudo-)distances between distributions and hypotheses used in this paper are defined as follows. They have close relationships with the privacy loss and are thus helpful in approximating differential privacy.

\begin{definition}[KL Divergence; cf. \cite{kullback1951information}]
        \label{def:KL_divergence}
        Suppose two distributions $P$ and $Q$ are defined on the same support. Then the KL divergence between $P$ and $Q$ is defined as
\begin{equation*}
    D_{KL}(P\Vert Q) =\mathbb E_P \left(\log\frac{\text d P}{\text d Q}\right).
\end{equation*}
\end{definition}

Here, we slightly abuse the notations of distribution and its cumulative distribution function when no ambiguity is introduced because there is a one-one mapping between them if we ignore zero-probability events.

\begin{definition}[Max Divergence; cf. \cite{dwork2014algorithmic}, Definition 3.6]
\label{def:max_div}
	For any random variables $X$ and $Y$, the max divergence  between $X$ and $Y$ is defined as
	\begin{equation*}
	D_{\infty}(X \| Y)=\max _{S \subseteq \operatorname{Supp}(X)}\left[\log \frac{\mathbb{P}(X \in S)}{\mathbb{P}(Y \in S)}\right].
	\end{equation*}
\end{definition}

\begin{definition}[$\delta$-Approximate Max Divergence; cf. \cite{dwork2014algorithmic}, Definition 3.6]
\label{def:max_div}
	For any random variables $X$ and $Y$, the $\delta$-approximate max divergence between $X$ to $Y$ is defined as
	\begin{equation*}
	D_{\infty}^{\delta}(X \| Y)=\max _{S \subseteq \operatorname{Supp}(X): \mathbb{P}(Y \in S) \geq \delta}\left[\log \frac{\mathbb{P}(X \in S)-\delta}{\mathbb{P}(Y \in S)}\right].
	\end{equation*}
\end{definition}

\begin{definition}[Statistical Distance; cf. \cite{dwork2014algorithmic}]
\label{def:stat_dis}
		For any random variables $X$ and $Y$, the statistical distance between $X$ and $Y$ is defined as
		\begin{equation*}
		\Delta(X \| Y)=\max_{S} \vert \mathbb{P}(X\in S)-\mathbb{P}(Y\in S)  \vert.
		\end{equation*}
\end{definition}

We then recall the following two lemmas.

\begin{lemma}[cf. \cite{dwork2016concentrated}, Lemmas 3.9 and 3.10]
	\label{lemma: antipodal}
	For any two distributions $\mathcal D$ and $\mathcal D'$, there exist  distributions $\mathcal M$ and $\mathcal M'$ such that
	\begin{align*}
		& \max\{D_{\infty}(\mathcal M\Vert \mathcal M'),D_{\infty}(\mathcal M'\Vert \mathcal M)\}
		= \max\{D_{\infty}(\mathcal D\Vert \mathcal D'), D_{\infty}(\mathcal D'\Vert \mathcal D)\},
	\end{align*}
	and
	\begin{equation*}
		D_{KL}(\mathcal D\Vert \mathcal D') \le D_{KL}(\mathcal M\Vert \mathcal M')=D_{KL}(\mathcal M'\Vert \mathcal M).
	\end{equation*}
\end{lemma}

\begin{lemma}[cf. \cite{dwork2014algorithmic}, Theorem 3.17]
\label{lemma:bridge}
For any random variables $Y$ and $Z$, we have that
\begin{gather*}
D^\delta_{\infty}(Y\|Z) \le \varepsilon,~
D^\delta_{\infty}(Z\| Y) \le \varepsilon,
\end{gather*}
      if and only if there exist random variables $Y'$, $Z'$ such that
\begin{gather*}
\Delta(Y\|Y')\le \frac{\delta}{e^\varepsilon+1},~
\Delta (Z\|Z')\le \frac{\delta}{1+e^\varepsilon},\\
D_{\infty}(Y'\|Z')\le \varepsilon,~
D_{\infty}(Z'\|Y')\le \varepsilon.
\end{gather*}
\end{lemma}

We finally recall Azuma Lemma \cite{boucheron2013concentration} which gives a concentration inequality to martingales.

\begin{lemma}[Azuma Lemma; cf. \cite{mohri2018foundations}, p. 371]
	\label{lemma:Azuma}
	Suppose $\{Y_i\}_{i=1}^T$ is a sequence of random variables, where $Y_i \in [-a_i,a_i]$. Let $\{X_i\}_{i=1}^T$ be a sequence of random variables such that,
	\begin{equation*}
	\mathbb E(Y_i\vert X_{i-1},..., X_{1}) \le C_i,
	\end{equation*}
	where $\{C_i\}_{i=1}^T$ is a sequence of constant real numbers. Then, we have the following inequality,
	\begin{equation*}
	\mathbb P\left(\sum\limits_{i=1}^T Y_i \ge \sum\limits_{i=1}^T C_i +t\sqrt{\sum\limits_{i=1}^T a_i^2}\right)\le e^{-\frac{t^2}{2}}.
	\end{equation*}
\end{lemma}

\section{Generalization Bounds for Iterative Differentially Private Algorithms}
\label{sec:main_result}

This section establishes the generalizability of iterative differentially private algorithms. The establishment has two steps. We first establish generalization bounds for any differentially private learning algorithm. Then, we investigate how the iterative nature shared by most learning algorithms would influence the differential privacy and further the generalizability via three composition theorems. We also sketch the proofs for these results and demonstrate their advantages compared with the existing results. %The two stages collectively establish the generalizability of iterative differentially private algorithms. %A final theorem is presented in the end of this section.

\subsection{Bridging Generalization and Privacy Preservation}
\label{sec:generalization_privacy}

We first prove a high-probability generalization bound for any $(\varepsilon, \delta)$-differentially private machine learning algorithm as follows. %, which upper bounds the cumulative distribution function of the generalization error $\mathcal R(\mathcal A(S)) - \hat{\mathcal R}_S (\mathcal A(S))$.

\begin{theorem}[High-Probability Generalization Bound via Differential Privacy]
\label{thm:high_probability_privacy}
Suppose algorithm $\mathcal{A}$ is $(\varepsilon,\delta)$-differentially private, the training sample size $N\ge\frac{2}{\varepsilon^{2}} \ln \left(\frac{16}{e^{-\varepsilon}\delta}\right)$, and the loss function $\Vert l\Vert_{\infty}\le 1$. Then, for any data distribution $\mathcal{D}$ over data space $\mathcal{Z}$, we have the following inequality,
\begin{equation*}
\mathbb{P}\left[\left|\hat{\mathcal{R}}_S(\mathcal{A}(S)) - \mathcal{R}(\mathcal{A}(S))\right| < 9\varepsilon\right] > 1-\frac{e^{-\varepsilon}\delta}{\varepsilon} \ln \left(\frac{2}{\varepsilon}\right).
\end{equation*}
\end{theorem}

Theorem \ref{thm:high_probability_privacy} demonstrates that a good privacy-preserving ability implies a good generalizability. Thus, we can unify the algorithm designing for enhancing privacy preservation and for improving generalization.

Theorem \ref{thm:high_probability_privacy} also implies a probably approximately correct (PAC)-learnable guarantee for $(\varepsilon, \delta)$-differentially private algorithms. PAC-learnability is defined as below,

\begin{definition}[PAC-Learnability; cf. \cite{mohri2018foundations}, Definition 2.4]
A concept class $\mathcal C$ is said to be PAC-learnable if there exists an algorithm $\mathcal A$ and a polynomial function $poly(\cdot, \cdot, \cdot, \cdot)$ such that for any $s > 0$ and $t > 0$, for all distributions $\mathcal D$ on the training example $Z$, any target concept $c \in \mathcal C$, and any sample size
\begin{equation*}
m \ge poly(1/s, 1/t, n, size(C)),
\end{equation*}
the following inequality holds,
\begin{equation*}
\mathbb P_{S \sim \mathcal D^m}(\mathcal R (\mathcal A(S)) < s) > 1 - t.
\end{equation*}
\end{definition}

In Section \ref{sec:applications}, we show how our result leads to PAC-learnable guarantees by using SGLD and agnostic federated learning as examples. %The PAC-learnable guarantee for the Laplacian mechanism, another major stream of privacy-preserving algorithms, can be similarly obtained.

\subsubsection{Proof Skeleton}

We now give the proof skeleton for Theorem \ref{thm:high_probability_privacy}. Please refer to Appendix \ref{proof:privacy-generalization} for more details. The proofs have three stages: (1) we first prove an on-average generalization bound for multi-database learning algorithms; (2) we then obtain a high-probability generalization bound for multi-database algorithms; and (3) we eventually prove Theorem \ref{thm:high_probability_privacy} by reduction to absurdity.

\textbf{Stage 1: Prove an on-average generalization bound for multi-database learning algorithms.}

%Theorems \ref{thm:high_probability_privacy} is established on a novel on-average generalization bound for {\it multi-database algorithms}.

We first prove the following on-average generalization bound for multi-database learning algorithms which are defined as follows.

\begin{definition}[Multi-Database Learning Algorithms; cf. \cite{nissim2015generalization}]
Suppose the training sample set $S$ is separated to $k$ sub-databases $S_1, \ldots, S_k$, each of which has the size of $N$. For the brevity, we rewrite the training sample set as below
\begin{equation*}
\vec{S} = (S_1, \ldots, S_k).
\end{equation*}
The hypothesis $\tilde{\mathcal A}(\vec{S})$ learned by {\it multi-database algorithm} $\tilde{\mathcal A}$ on dataset $\vec{S}$ is defined as follows,
\begin{equation*}
\tilde{\mathcal A}(\vec{S}): \mathcal Z^{kN} \mathcal H \times \{1, \ldots, k\},~ \vec{S} \mapsto \left(h_{\mathcal A(\vec{S})}, i_{\mathcal A(\vec{S})} \right).
\end{equation*}
\end{definition}

\begin{theorem}[On-Average Multi-Database Generalization Bound]
	\label{thm:multi_database}
	Let algorithm,
	\begin{equation*}
	\tilde{\mathcal{A}}: \vec{S} \to \mathcal H\times \{1,\cdots,k\},
	\end{equation*}
	is $(\varepsilon,\delta)$-differentially private and the loss function $\Vert l\Vert_{\infty}\le 1$. Then, for any data distribution $\mathcal{D}$ over data space $\mathcal{Z}$, we have the following inequality,
	% any database set $\vec{S}=\{S_i\}_{i=1}^k$ where $S_i$ is a database contains $N$ i.i.d. sample from $\mathcal{D}$, if we denote $\left(h_{\mathcal{A}(\vec{S})},i_{\mathcal{A}(\vec{S})}\right)=\mathcal{A}(\vec{S})$, we have 
	\begin{align}
	\label{eq:stab_databases}
	& \left|\underset{\vec{S} \sim \mathcal{D}^{N}}{\mathbb{E}}\left[\underset{ \mathcal{A}(\vec{S})}{\mathbb{E}}\left[\hat{\mathcal{R}}_{S_{i_{\mathcal{A}(\vec{S})}}}\left(h_{\mathcal{A}(\vec{S})}\right)\right] - \underset{ \mathcal{A}(\vec{S})}{\mathbb{E}}\left[\mathcal{R}\left(h_{\mathcal{A}(\vec{S})}\right)\right]\right]\right| 
	\leq e^{-\varepsilon}k \delta+1-e^{-\varepsilon}.
	\end{align}
\end{theorem}

Since $1-e^{-\varepsilon}\le\varepsilon$, we have the following corollary.

	\begin{corollary}
		\label{coro:stab_multi}
		Suppose all the conditions in Theorem \ref{thm:multi_database} hold, then  we have the following inequality,
		\begin{align*}
		&\underset{\vec{S} \sim \mathcal{D}^{kN}}{\mathbb{E}}\left[\underset{ \mathcal{A}(\vec{S})}{\mathbb{E}}\left[\hat{\mathcal{R}}_{S_{i_{\mathcal{A}(\vec{S})}}}(h_{\mathcal{A}(\vec{S})})\right]\right] %\nonumber\\
		\le e^{-\varepsilon}k \delta+\varepsilon+\underset{\vec{S} \sim \mathcal{D}^{kN}}{\mathbb{E}}\left[\underset{ \mathcal{A}(\vec{S})}{\mathbb{E}}\left[\mathcal{R}\left(h_{\mathcal{A}(\vec{S})}\right)\right]\right].
		\end{align*}
	\end{corollary}

\textbf{Stage 2: Prove a high-probability generalization bound for multi-database algorithms.}

Markov bound (cf. \cite{mohri2018foundations}, Theorem C.1) is an important concentration inequality in learning theory. Here, we slightly modify the original version as follows,
\begin{equation*}
\mathbb{E}_{x}\left[h(x)\right]\ge \mathbb{E}_{x}\left[h(x)\mathbb{I}_{h(x)\ge g(x)}\right]\ge \mathbb{E}_{x}\left[g(x)\mathbb{I}_{h(x)\ge g(x)}\right].
\end{equation*}

Then, combining it with Theorem \ref{thm:multi_database}, we derive the following high-probability generalization bound for multi-database algorithms.

	\begin{theorem}[High-Probability Multi-Database Generalization Bound]
		\label{lem:high_multi}
		Let the following algorithm,
		\begin{equation*}
		\mathcal{A}: \mathcal{Z}^{kN}\rightarrow \mathcal{Y}^\mathcal{X}\times \{1,\cdots,k\},~ \vec{S} \mapsto \left(h_{\mathcal{A}(\vec{S})},i_{\mathcal{A}(\vec{S})}\right),
		\end{equation*}
		be $(\varepsilon,\delta)$-differential private, where $kN$ is the size of the whole dataset $\vec{S}$ and $\mathcal{Y}^\mathcal{X} = \{ f: \mathcal{X} \to \mathcal{Y}\}$.
		Then, for any data distribution $\mathcal{D}$ over data space $\mathcal{Z}$, any database set $\vec{S}=\{S_i\}_{i=1}^k$, where $S_i$ is a database contains $N$ i.i.d. examples drawn from $\mathcal{D}$, we have the following generalization bound,
		\begin{equation}
		\label{eq:generalization_bound_conflict}
		\mathbb{P}\left[\hat{\mathcal{R}}_{S_{i_{\mathcal{A}(\vec{S})}}}\left(h_{\mathcal{A}(\vec{S})}\right) \leq \mathcal{R}\left(h_{\mathcal{A}(\vec{S})}\right)+k e^{-\varepsilon} \delta+ 3\varepsilon\right] \geq \varepsilon.
		\end{equation}
	\end{theorem}

\textbf{Stage 3: Prove Theorem \ref{thm:high_probability_privacy} by Reduction to Absurdity.}

We eventually prove Theorem \ref{thm:high_probability_privacy} by {\it reduction to absurdity}. Assume there exists an algorithm $\mathcal{A}$ which conflicts with Theorem \ref{thm:high_probability_privacy}. We can then construct an algorithm $\mathcal{B}$ based on the exponential mechanism which is defined as follows.
\begin{definition}[Exponential Mechanism; cf. \cite{nissim2015generalization}, p. 3, and \cite{mcsherry2007mechanism}]
Suppose that $S$ is a sample set, $u: (S, r) \mapsto \mathbb R^+$ is the utility function, $R$ is an index set, $\varepsilon$ is the privacy parameter, and $\Delta u$ is the sensitivity of $u$ defined by
	 \begin{equation*}
	 \Delta u \overset{\Delta}{=} \max _{r \in R} \max _{S, S'\text{ adjacent} }|u(S, r)-u(S', r)|.
	 \end{equation*}
Then, the exponential mechanism $q(S,u,R,\varepsilon)$ is defined as $(S,u,R,\varepsilon) \mapsto r$, where $r \in R$.
\end{definition}
%Let the input be $\vec{S}=(S_1,\cdots,S_k)$ and $T$, where $S_i, T\in \mathcal{Z}^N$. We then run the exponential mechanism with the utility function defined as follows,
%\begin{equation*}
%u(\vec{S},T,i)=N\hat{\mathcal{R}}_{S_i}(\mathcal{A}(S_i))-N\hat{\mathcal{R}}_{T}(\mathcal{A}(S_i)).
%\end{equation*}
%We can prove algorithm $\mathcal{B}$ is $(2\varepsilon,\delta)$-differential private. Also,
%\begin{equation*}
%\mathbb{P}\left[\hat{\mathcal{R}}_{S_{i_{\mathcal{B}(\vec{S})}}}\left(h_{\mathcal{B}(\vec{S})}\right) \leq \mathcal{R}\left(h_{\mathcal{B}(\vec{S})}\right)+k e^{-\varepsilon} \delta+ 3\varepsilon\right] < \varepsilon.
%\end{equation*}
%This stage is summarized as the following lemma.

Then, we can prove the following lemma.
	
	\begin{lemma}
	\label{lemma:counterexample}
		We define an algorithm $\mathcal{A}: \mathcal{Z}^N\rightarrow \mathcal{Y}^\mathcal{X}$, where $N$ is the training sample size, $\mathcal Z$ is the data space, $\mathcal{Z}^N$ is the space of training sample set, and $\mathcal{Y}^\mathcal{X} = \{ f: \mathcal{X} \to \mathcal{Y} \}$. Suppose $k= \ceil{\frac{\varepsilon}{e^{-\varepsilon}\delta}}$ and
		\begin{equation*}
		N\ge\frac{2}{\varepsilon^{2}} \ln \left(\frac{16}{e^{-\varepsilon}\delta}\right).
		\end{equation*}
		If we have that 
		\begin{align}
		\label{eq:condition_gene}
		& \mathbb{P}\left[\hat{\mathcal R}(\mathcal{A}(S)) \leq e^{-\varepsilon}k \delta+ 8\varepsilon+{\mathcal R}(\mathcal{A}(S))\right]  %\nonumber\\
		< 1-\frac{e^{-\varepsilon}\delta}{\varepsilon} \ln \left(\frac{2}{\epsilon}\right),
		\end{align}
		then there exists an algorithm
		\begin{equation*}
		\mathcal{B}:~\mathcal{Z}^{kN}\rightarrow \mathcal{Y}^\mathcal{X}\times \{1,\cdots,k\},
		\end{equation*}
		is  $(2\varepsilon,\delta)$-differentially private and
		\begin{equation}
		\label{eq:conflict_2}
		\mathbb{P}\left[\hat{\mathcal{R}}_{S_{i_{\mathcal{B}(\vec{S})}}}\left(h_{\mathcal{B}(\vec{S})}\right) \leq \mathcal{R}\left(h_{\mathcal{B}(\vec{S})}\right)+k e^{-\varepsilon} \delta+ 3\varepsilon\right] < \varepsilon,
		\end{equation}
		where $\vec{S}=\{S_i\}_{i=1}^k$ and $S_i$ is a database contains $N$ i.i.d. sample from $\mathcal D$.
	\end{lemma}
	
	The eq. (\ref{eq:conflict_2}) in Lemma \ref{lemma:counterexample} conflicts with eq. (\ref{eq:generalization_bound_conflict}). Thus, we proved Theorem \ref{thm:high_probability_privacy}.

\subsubsection{Comparison with Existing Results}
\label{sec:generalization_comparison}

This section compares our results with the existing works.

\textbf{Comparison of Theorem \ref{thm:high_probability_privacy}.}

There have been several high-probability generalization bounds for $(\varepsilon, \delta)$-differentially private machine learning algorithms.

Dwork et al. \cite{dwork2015preserving} proved that
\begin{equation*}
\mathbb P\left[\mathcal R(\mathcal A(S)) - \hat{\mathcal R}_S (\mathcal A(S)) < 4\varepsilon\right] > 1- 8\delta^\varepsilon.
\end{equation*}

Oneto et al. \cite{oneto2017differential} proved that
\begin{align*}
& \mathbb P\left[ \text{Diff } \mathcal R < \sqrt{6 \hat{ \mathcal{R}}_S(\mathcal{A}(S)) }\hat \varepsilon+6\left(\varepsilon^{2}+1 / N\right)\right]%\\
> 1 - 3 e^{-N \varepsilon^{2}},
\end{align*}
and
\begin{align*}
& \mathbb P\left[\text{Diff } \mathcal R < \sqrt{4 \hat{V}_S(\mathcal{A}(S))} \hat \varepsilon + \frac{5 N}{N-1}\left(\varepsilon^{2}+1 / N\right)\right] %\\
> 1 - 3 e^{-N \varepsilon^{2}},
\end{align*}
where
\begin{gather*}
\text{Diff } \mathcal R = \mathcal R(\mathcal A(S)) - \hat{\mathcal R}_S (\mathcal A(S)),\\
\hat \varepsilon = \varepsilon+\sqrt{1 / N},
\end{gather*}
and $\hat{V}_S(\mathcal{A}(S))$ is the empirical variance of $l(\mathcal{A}(S), \cdot)$:
\begin{gather*}
\hat{V}_S(\mathcal{A}(S)) = \frac{1}{2N(N-1)}\sum_{i\ne j} \left[\ell\left(\mathcal{A}(S), z_{i}\right)-\ell\left(\mathcal{A}(S), z_{j}\right)\right]^{2}.
\end{gather*}

Nissim and Stemmer \cite{nissim2015generalization} proved that
\begin{equation*}
\mathbb P\left[\mathcal R(\mathcal A(S)) - \hat{\mathcal R}_S (\mathcal A(S)) <13\varepsilon\right] > 1 - \frac{2 \delta}{\varepsilon} \log \left(\frac{2}{\varepsilon}\right).
\end{equation*}
This is the existing tightest high-probability generalization bound in the literature. However, this bound only stands for binary classification problems. By contrast, our high-probability generalization bound holds for any machine learning algorithm.

Also, our bound is strictly tighter. All the bounds, including ours, are in the following form, 
\begin{equation*}
\mathbb P\left[\mathcal R(\mathcal A(S)) - \hat{\mathcal R}_S (\mathcal A(S)) < a \right] > 1- b,
\end{equation*}
where $a$ and $b$ are two positive constant real numbers. Apparently, a smaller $a$ and a smaller $b$ imply a tighter generalization bound. Our bound improves the current tightest result from two aspects:
\begin{itemize}
\item
Our bound tightens the term $a$ from $13\varepsilon$ to $9\varepsilon$.

\item
Our bound tightens the term $b$ from $\frac{2 \delta}{\varepsilon} \log \left(\frac{2}{\varepsilon}\right)$ to $\frac{2 e^{-\varepsilon}\delta}{\varepsilon} \log \left(\frac{2}{\varepsilon}\right)$.
\end{itemize}
These improvements are significant. Adabi et al. \cite{abadi2016deep} conducted experiments on the differential privacy in deep learning. Their empirical results demonstrate that the factor $\varepsilon$ can be as large as $10$.

%Recently, Jung et al. \cite{jung2019new} defined two new measures to evaluate the generalizability, i.e., in-sample accuracy and distributional accuracy, upon which it also developed the relationship between generalization and privacy preservation. Specifically, it proves high-probability bounds on the distributional accuracy of any mechanism that is both differentially private and has a bounded in-sample accuracy. It is an interesting open problem to discover the linkage of the conventional generalization error with the in-sample accuracy and the out-of-sample accuracy.

\textbf{Comparison of Theorem \ref{thm:multi_database}.}

There is only one related work in the literature that presents an on-average generalization bound for multi-database algorithm. Nissim and Stemmer \cite{nissim2015generalization} proved that,
\begin{align*}
	& \left|\underset{\vec{S} \sim \mathcal{D}^{kN}}{\mathbb{E}}\left[\underset{ \mathcal{A}(\vec{S})}{\mathbb{E}}\left[\hat{\mathcal{R}}_{S_{i_{\mathcal{A}(\vec{S})}}}(h_{\mathcal{A}(\vec{S})})\right] - \underset{ \mathcal{A}(\vec{S})}{\mathbb{E}}\left[\mathcal{R}\left(h_{\mathcal{A}(\vec{S})}\right)\right]\right]\right| %\nonumber\\
	\le k\delta+2\varepsilon.
\end{align*}
Our bound is tighter by a factor of $e^\varepsilon$. According to the empirical results by Adabi et al. \cite{abadi2016deep}, this factor can be as large as $e^{10} \approx 20,000$. It is a significant multiplier for loss function. Furthermore, the result by Nissim and Stemmer stands only for binary classification, while our result apply to all differentially private learning algorithms.

\subsection{How the Iterative Nature Contributes?}
\label{sec:generalization_iterative_private}

Most machine learning algorithms are iterative, which may degenerate the privacy-preserving ability along with iterations. %This degeneration is ignored in the previous section as well as most exisiting works on the generalization-privacy relationship. 
This section studies the degenerative nature of the privacy preservation in iterative machine learning algorithms and its influence to the generalization. %Specifically, we prove a novel composition theorem that calculates the differential privacy of any iterative machine learning algorithm via the differential privacy of its iterations. Combined with the results in the previous section, this composition theorem delivers generalization bounds of iterative machine learning algorithms. %We first define iterative machine learning algorithms as follows.

%\subsubsection{Composition Theorems}

We have the following composition theorem.
   
\begin{theorem}[Composition Theorem I]
\label{thm:composition_2}
Suppose an iterative machine learning algorithm $\mathcal A$ has ${T}$ steps: $\left\{Y_i(S)\right\}_{i=0}^T$, where $Y_i$ is the learned hypothesis after the $i$-th iteration. Suppose the $i$-th iterator
\begin{equation*}
M_i: (Y_{i-1}, S) \mapsto Y_i
\end{equation*}
is $(\varepsilon,\delta)$-differentially private. Then, the algorithm $\mathcal A$ is $(\varepsilon',\delta')$-differentially private. The factor $\varepsilon'$ is as follows,
	\begin{gather}
	\label{eq:iterative_epsilon}
	\varepsilon'=\min \left\{ \varepsilon'_1, \varepsilon'_2, \varepsilon'_3 \right \},%\\
		%\label{eq:iterative_delta}
	%\delta'=\max_{\{\alpha_i\}_{i=1}^T \in I} \left \{ 1-\prod_{i=1}^T\left(1-e^{\alpha_i}\frac{\delta_i}{1+e^{\varepsilon_i}}\right)+1-\prod_{i=1}^T\left(1-\frac{\delta_i}{1+e^{\varepsilon_i}}\right)+\tilde\delta \right \}, \nonumber
	\end{gather}
	where 
	\begin{gather}
	\varepsilon'_1 = \sum_{i=1}^{T} \varepsilon_{i},\\
	\varepsilon'_2 = \sum_{i=1}^{T} \frac{\left(e^{\varepsilon_{i}}-1\right) \varepsilon_{i}}{e^{\varepsilon_{i}}+1}+\sqrt{2\sum_{i=1}^{T} \varepsilon_{i}^{2} \log \left(e+\frac{\sqrt{\sum_{i=1}^{T} \varepsilon_{i}^{2}}}{\tilde{\delta}}\right)},\nonumber\\
	\label{eq:epsilon_3}
	\varepsilon'_3 = \sum_{i=1}^{T} \frac{\left(e^{\varepsilon_{i}}-1\right) \varepsilon_{i}}{e^{\varepsilon_{i}}+1}+\sqrt{2 \log \left(\frac{1}{\tilde{\delta}}\right) \sum_{i=1}^{T} \varepsilon_{i}^{2}},
	\end{gather}
	and $\tilde \delta$ is an arbitrary positive real constant.
	
	Correspondingly, the factor $\delta'$ is defined as the maximal value of the following equation with respect to $\{\alpha_i\}_{i=1}^T \in I$,
	\begin{equation}
	\label{eq:iterative_delta}
	1-\prod_{i=1}^T\left(1-e^{\alpha_i}\frac{\delta_i}{1+e^{\varepsilon_i}}\right)+1-\prod_{i=1}^T\left(1-\frac{\delta_i}{1+e^{\varepsilon_i}}\right)+\tilde\delta,
	\end{equation}
	where
	\begin{equation*}
	I=\left\{\{\alpha_i\}_{i=1}^T: \sum_{i=1}^T \alpha_i=\varepsilon',~ |\{i:\alpha_i\ne \varepsilon_i,~ \alpha_i\ne 0\}|\le 1\right\},
	\end{equation*}
	and $\tilde \delta$ is the same real constant mentioned above.
\end{theorem}

%{\color{red}
%This theorem is derived by constructing two mechanisms with small distance respectively from $\{Y_i(S)\}_{i=0}^T$ and $\{Y_i(S')\}_{i=0}^T$ while there are pure differential privacy between them.  
%}

When all the iterations have the same privacy-preserving ability, we can tighten the approximation of the factor $\delta'$ as the following corollary.

\begin{corollary}[Composition Theorem II]
When all the iterations are $(\varepsilon, \delta)$-differential private, $\delta'$ is
\begin{align*}
\delta' = & 1-\left(1-e^{\varepsilon}\frac{\delta}{1+e^{\varepsilon}}\right)^{\left \lceil  \frac{\varepsilon'}{\varepsilon}\right \rceil}\left(1-\frac{\delta}{1+e^{\varepsilon}}\right)^{T-\left \lceil  \frac{\varepsilon'}{\varepsilon}\right \rceil} + 1 %\\
 - \left(1-\frac{\delta}{1+e^{\varepsilon}}\right)^{T}+\tilde{\delta}\\
= & \left(T-\left\lceil\frac{\varepsilon'}{\varepsilon}\right\rceil\right)\frac{2\delta}{1+e^\varepsilon}+\left\lceil\frac{\varepsilon'}{\varepsilon}\right\rceil\delta+\tilde\delta %\\
+ \mathcal O \left(\left(\frac{\delta}{1+e^\varepsilon}\right)^2\right).
\end{align*}

\begin{proof}
The maximum of $\delta'$ is achieved when at most $T-\left\lceil\frac{\varepsilon'}{\varepsilon}\right\rceil$ elements $\alpha_i \ne 0$. We note that
\begin{equation*}
(1-x)^n=1-nx+\mathcal O(x^2).
\end{equation*}
Then, the $\delta'$ in Theorem \ref{thm:composition_2} can be estimated as 

\begin{align*}
\delta' =& 1-\left(1-e^{\varepsilon}\frac{\delta}{1+e^{\varepsilon}}\right)^{\left \lceil  \frac{\varepsilon'}{\varepsilon}\right \rceil}\left(1-\frac{\delta}{1+e^{\varepsilon}}\right)^{T-\left \lceil  \frac{\varepsilon'}{\varepsilon}\right \rceil}+1 %\\
-\left(1-\frac{\delta}{1+e^{\varepsilon}}\right)^{T}+\tilde{\delta}
\\
=&1+T\frac{\delta}{1+e^\varepsilon}+\tilde\delta +O\left(\left(\frac{\delta}{1+e^\varepsilon}\right)^2\right)\\
&-\left(1-\left\lceil\frac{\varepsilon'}{\varepsilon}\right\rceil\frac{\delta}{1+e^\varepsilon}+O\left(\left(\frac{\delta}{1+e^\varepsilon}\right)^2\right)\right)%\\
 \left(1-\left(T-\left\lceil\frac{\varepsilon'}{\varepsilon}\right\rceil\right)\frac{\delta}{1+e^\varepsilon}+O\left(\left(\frac{\delta}{1+e^\varepsilon}\right)^2\right)\right)\\
=&
\left\lceil\frac{\varepsilon'}{\varepsilon}\right\rceil\frac{\delta}{1+e^\varepsilon}+\left(T-\left\lceil\frac{\varepsilon'}{\varepsilon}\right\rceil\right)\frac{\delta}{1+e^\varepsilon}+ T\frac{\delta}{1+e^\varepsilon} %\\
 +O\left(\left(\frac{\delta}{1+e^\varepsilon}\right)^2\right)  +\tilde\delta + O\left(\left(\frac{\delta}{1+e^\varepsilon}\right)^2\right)
\\
\approx&\left(T-\left\lceil\frac{\varepsilon'}{\varepsilon}\right\rceil\right)\frac{2\delta}{1+e^\varepsilon}+\left\lceil\frac{\varepsilon'}{\varepsilon}\right\rceil\delta+\tilde\delta.
\end{align*}

The proof is completed.
\end{proof}

\end{corollary}

When all the iterators $M_i$ are $\varepsilon$-differentially private, we can further tighten the third estimation of $\varepsilon'$ in Theorem \ref{thm:composition_2}, eq. (\ref{eq:iterative_epsilon}) as the following composition theorem.

\begin{corollary}[Composition Theorem III]
\label{thm:composition_moment}
Suppose all the iterators $M_i$ are $\varepsilon$-differentially private and all the other conditions in Theorem \ref{thm:composition_2} hold. Then, the algorithm $\mathcal A$ is $(\varepsilon',\delta')$-differentially private that
\begin{align*}
\varepsilon' = & T \frac{\left(e^{\varepsilon}-1\right) \varepsilon}{e^{\varepsilon}+1}+\sqrt{2 \log \left(\frac{1}{\tilde{\delta}}\right) T \varepsilon^{2}},
\\
\delta' = & 1-\left(1-e^{\varepsilon}\frac{\delta}{1+e^{\varepsilon}}\right)^{\left \lceil  \frac{\varepsilon'}{\varepsilon}\right \rceil}\left(1-\frac{\delta}{1+e^{\varepsilon}}\right)^{T-\left \lceil  \frac{\varepsilon'}{\varepsilon}\right \rceil} \nonumber +1 %\\
 -\left(1-\frac{\delta}{1+e^{\varepsilon}}\right)^{T}+\delta'',
\end{align*}
where $ \delta''$ is defined as follows:

\begin{equation*}
 \delta'' =e^{-\frac{\varepsilon'+T\varepsilon}{2}}\left(\frac{1}{1+e^\varepsilon}\left(\frac{2T\varepsilon}{T\varepsilon-\varepsilon'}\right)\right)^T\left(\frac{T\varepsilon+\varepsilon'}{T\varepsilon-\varepsilon'}\right)^{-\frac{\varepsilon'+T\varepsilon}{2\varepsilon}}.
\end{equation*}
\end{corollary}

%\paragraph{Generalization bounds for iterative differentially private algorithms.}
%\begin{remark}
The three composition theorems extend the developed relationship between generalization and privacy preservation to iterative machine learning algorithms. At this point, we establish the theoretical foundation for the generalizability of iterative differentially private machine learning algorithms.
%\end{remark}

\subsubsection{Proof Skeleton}

We now sketch the proofs for Theorem \ref{thm:composition_2}. Please refer to Appendix \ref{proof:composition} for more details. We also proved two additional composition theorems as by-products. The two composition theorems are weaker than Theorem \ref{thm:composition_2} but play essential roles in the proofs. The proofs have four stages: (1) we first approximate the KL-divergence between hypotheses learned on neighboring training sample sets; (2) we then prove a composition bound for $\varepsilon$-differentially private learning algorithms; (3) this composition theorem is improved to a composition bound for $(\varepsilon, \delta)$-differentially private learning algorithms; and (4) we eventually tighten the result in (2) to obtain Theorem 4.

%Differential privacy measures the distance between the hypotheses learned from neighboring training sample sets which are different by only one single example. 

\textbf{Stage 1: Approximate the KL-divergence between hypotheses learned on neighboring training sample sets.}

It would be technically difficult to approach direcyly the differential privacy of an iterative learning algorithm from the differential privacy of every iteration. To relieve the technical difficulty, we employ KL divergence as a bridge in this paper. %, which is defined as follows.
For any $\varepsilon$-differentially private learning algorithm, we prove the following lemma to approximate the KL-divergence between hypotheses learned on neighboring training sample sets.

%For any $\varepsilon$-differentially private learning algorithm, we prove the following lemma to approximate the KL-divergence between hypotheses learned on neighboring training sample sets.

\begin{lemma}
		\label{lemma: relation of kl}
		If $\mathcal{A}$ is an $\varepsilon$-differentially private algorithm , then for every neighbor database pair $S$ and $S'$, the KL divergence between hypotheses $\mathcal{A}(S)$ and $\mathcal{A}(S')$ satisfies the following inequality,
		\begin{equation*}
		D_{KL}(\mathcal{A}(S)\Vert \mathcal{A}(S')) \le \varepsilon \frac{e^{\varepsilon}-1}{e^{\varepsilon}+1}.	 
		\end{equation*}
	\end{lemma}
	
This lemma is novel and the proof is technically non-trivial. Lemma \ref{lemma: antipodal} helps establish the proof of Lemma \label{lemma: relation of kl}.

There are two related results in the literature, which are considerably looser than ours. Dwork et al. \cite{dwork2010boosting} proved %under the same assumption, 
an inequality of the KL divergence as follows,
 \begin{equation*}
 D_{KL}(\mathcal{A}(S)\Vert \mathcal{A}(S')) \le \varepsilon (e^{\varepsilon}-1).
 \end{equation*}
Then, Dwork and Rothblum \cite{dwork2016concentrated} further improved it to
    \begin{equation}
    \label{eq:KL_existing_optimal}
    D_{KL}(\mathcal{A}(S)\Vert \mathcal{A}(S')) \le\frac{1}{2} \varepsilon (e^{\varepsilon}-1).
    \end{equation}
Compared with ours, eq. (\ref{eq:KL_existing_optimal}) is larger by a factor $(1+e^\varepsilon)/2$, which can be very large in practice.

\textbf{Stage 2: Prove a weaker composition theorem where every iteration is $\varepsilon$-differential private.}

Based on Lemma \ref{lemma: relation of kl}, we can prove the following composition theorem as a preparation theorem.

\begin{theorem}[Composition Theorem IV]
\label{thm:composition_IV}
	Suppose an iterative machine learning algorithm $\mathcal A$ has $T$ steps: $\left\{Y_i(S)\right\}_{i=1}^T$. Specifically, we define the $i$-th iterator as follows,
	\begin{equation}
	M_i: (Y_{i-1}(S), S) \mapsto Y_{i}(S).
	\end{equation}
	Assume that $Y_0$ is the initial hypothesis (which does not depend on $S$). If for any fixed observation $y_{i-1}$ of the variable $Y_{i-1}$, $M_i(y_{i-1},S)$ is $\varepsilon_i$-differentially private, 
	then $\left\{Y_i(S)\right\}_{i=0}^T$ is ($\varepsilon'$, $\delta'$)-differentially private that
	\begin{equation*}
	\varepsilon^{\prime}=\sqrt{2  \log \left( \frac{1}{\delta'}\right)\left(\sum\limits_{i=1}^T \varepsilon_{i}^2\right)} +\sum\limits_{i=1}^T \varepsilon_i \frac{e^{\varepsilon_i}-1}{e^{\varepsilon_i}+1}.
	\end{equation*}
\end{theorem}
	 
\textbf{Stage 2: Prove a weaker composition theorem where every iteration is $(\varepsilon_i, \delta_i)$-differentially private.}

Based on Lemmas \ref{lemma: antipodal} and \ref{lemma:Azuma}, we proved the following lemma that the maximum of the following function,
\begin{equation}
\label{eq:optimal_function}
f\left(\left\{\alpha_i\right\}_{i=1}^T\right)=1-\prod_{i=1}^T(1-\alpha_i A_i),
\end{equation}
is achieved when $\{\alpha_i\}_{i=1}^T$ are at the boundary under some restrictions.

%\label{subsubsec:composition_general}
\begin{lemma}
	\label{lemma:opt}
	The maximum of function (\ref{eq:optimal_function}) when $A_i$ is positive real such that,
	\begin{gather*}
	1\le \alpha_i \le c_i,\text{ (here } c_iA_i\le 1\text{)},~\text{and}~
	\prod_{i=1}^{T}\alpha_i=c,
	\end{gather*} is achieved at the point when the cardinality follows the inequality:
	 \begin{equation}
	 |\{i:\alpha_i\ne c_i \text{ and }\alpha_i\ne 1\}|\le 1.
	 \end{equation}
	 \end{lemma} 
	 
Based on Lemmas \ref{lemma:bridge} and \ref{lemma:opt}, we can prove the following composition theorem whose estimate of $\varepsilon'$ is somewhat looser than our main results.

\begin{theorem}[Composition Theorem V]
	\label{thm:composition_V}
		Suppose an iterative machine learning algorithm $\mathcal A$ has $T$ steps: $\left\{Y_i(S)\right\}_{i=1}^T$. Specifically, let the $i$-th iterator be as follows,
		\begin{equation}
		M_i: (Y_{i-1}(S), S) \mapsto Y_{i}(S).
		\end{equation}
		Assume that $Y_0$ is the initial hypothesis (which does not depend on $S$). If for any fixed observation $y_{i-1}$ of the variable $Y_{i-1}$, $M_i(y_{i-1},S)$ is $(\varepsilon_i, \delta_i)$-differentially private $(i\ge 1)$,
	then $\left\{Y_i(S)\right\}_{i=0}^T$ is $(\varepsilon', \delta')$-differentially private where
	\begin{align*}
	\varepsilon^{\prime} = & \sqrt{2  \log \left( \frac{1}{\tilde\delta}\right)\left(\sum\limits_{i=1}^T \varepsilon_{i}^2\right)} +\sum\limits_{i=1}^T \varepsilon_i \frac{e^{\varepsilon_i}-1}{e^{\varepsilon_i}+1},\nonumber\\
	\delta' = & \max_{\{\alpha_i\}_{i=1}^T \in I} 1-\prod_{i=1}^T\left(1-e^{\alpha_i}\frac{\delta_i}{1+e^{\varepsilon_i}}\right)+ 1 %\\
	- \prod_{i=1}^T\left(1-\frac{\delta_i}{1+e^{\varepsilon_i}}\right)+\tilde\delta,
	\end{align*}
	and $I$ is defined as the set of $\{\alpha_i\}_{i=1}^T$ such that
	\begin{gather*}
	\sum_{i=1}^T \alpha_i=\varepsilon',~
	|\{i:\alpha_i\ne \varepsilon_i \text{ and }\alpha_i\ne 0\}|\le 1.
	\end{gather*}
\end{theorem}

%Then, for any iterative learning algorithm whose $i$-th iteration is $(\varepsilon_i, \delta_i)$-differentially private, we prove there exists an iterative learning algorithm $\tilde{\mathcal A}$ whose $i$-th iteration is $\varepsilon_i'$-differentially private and the distance between them is controlled by constants relying on $\varepsilon_i$ and $\delta_i$ measured by $D_\infty(X \| Y)$, $D_\delta(X \| Y)$, and $\Delta(X \| Y)$ defined as follows.

%Based on Theorem \ref{thm:composition_V}, we can calculate the $(\varepsilon', \delta')$-differential privacy of algorithm $\tilde{\mathcal A}$. Eventually, we can calculate the $(\varepsilon, \delta)$-differential privacy that algorithm $\mathcal A$ as a weaker version of Theorem \ref{thm:composition_2} in which $\varepsilon'$ is replaced by $\varepsilon_2$ wherein as follows. %; see Appendix \ref{subsubsec:general}, Theorem \ref{thm:weaker_composition_fomer}.

\textbf{Stage 4: Prove Theorem \ref{thm:composition_2}.}

Applying Lemma \ref{lemma:Azuma} and Theorem 3.5 in \cite{kairouz2015composition}, we eventually extend the weaker versions to Theorem \ref{thm:composition_2}. Theorem 3.5 in \cite{kairouz2017composition} relies on a term {\it privacy area} defined wherein. A larger privacy area corresponds to a worse privacy preservation. In this paper, we make a novel contribution that proves the moment generating function of the following random variable represents the worst case,
\begin{equation}
\log\left(\frac{\mathbb{P}(\cap_i Y_i(S)\in B_i)}{\mathbb{P}(\cap_i Y_i(S')\in B_i)}\right),
\end{equation}
where $Y_i( S)$ and $Y_i(S')$ are the mechanisms achieving the largest privacy area. Thus, we can deliver an approximation of the differential privacy via this moment generating function. %This moment generating function represents the worst case. We then obtain the composition theorem \ref{thm:composition_moment} by applying the Chernoff concentration inequality (see \cite{mohri2018foundations} for details). 

%This composition theorem is established based on a novel inequality of KL-divergence as follows,

\subsubsection{Comparison with Existing Results}

Our composition theorem is strictly tighter than the existing results.

A classic composition theorem is as follows (see \cite{dwork2014algorithmic}, Theorem 3.20 and Corollary 3.21, pp. 49-52),
\begin{gather*}
\varepsilon' = \sum\limits_{i=1}^T \varepsilon_i (e^{\varepsilon_i}-1) + \sqrt{2 \log \left( \frac{1}{\delta}\right) \sum\limits_{i=1}^T \varepsilon_{i}^2},\\
\delta' = \tilde \delta+\sum\limits_{i=1}^T\delta_i,
\end{gather*}
where $\tilde \delta$ is an arbitrary positive real number, $(\varepsilon', \delta')$ is the differential privacy of the whole algorithm, and $(\varepsilon_i, \delta_i)$ is the differential privacy of the $i$-th iteration.

Currently, the tightest approximation is given by Kairouz et al. \cite{kairouz2017composition} as follows,
\begin{gather*}
\label{eq:KOV_epsilon}
\varepsilon'=\min \left\{\varepsilon'_1, \varepsilon'_2, \varepsilon'_3 \right\},\\
\label{eq:KOV_delta}
\delta'=1-(1-\delta)^{T}(1-\tilde{\delta}),
\end{gather*}
where
\begin{gather*}
\varepsilon'_1 = T \varepsilon,\\
\varepsilon'_2 = \frac{\left(e^{\varepsilon}-1\right) \varepsilon T}{e^{\varepsilon}+1}+\varepsilon \sqrt{2 T \log \left(e+\frac{\sqrt{T \varepsilon^{2}}}{\tilde{\delta}}\right)},\\
\varepsilon'_3 = \frac{\left(e^{\varepsilon}-1\right) \varepsilon T}{e^{\varepsilon}+1}+\varepsilon \sqrt{2 T \log \left(\frac{1}{\tilde{\delta}}\right)}.
\end{gather*}
The estimate of the $\varepsilon'$ is the same as ours, while their $\delta'$ is also larger than ours approximately by
\begin{equation*}
\delta \frac{e^\varepsilon-1}{e^\varepsilon+1} \left(T-\left\lceil \frac{\varepsilon'}{\varepsilon}\right\rceil\right).
\end{equation*}
The iteration number $T$ is usually overwhelmingly large, which guarantees our advantage is significant. %For a detailed proof of the comparison, please refer to Section \ref{sec:approximation_delta}.

%Theorem \ref{thm:composition_2} provides a strictly tighter estimate of the $(\varepsilon', \delta')$-differential privacy than the existing works. %, which will help tighten the generalization bounds. %Our result for KL divergence is strictly better than existing outcomes. 

\section{Applications}
\label{sec:applications}

Our theories apply to a wide spectrum of machine learning algorithms. This section implements them to two popular regimes as examples: (1) stochastic gradient Langevin dynamics \cite{welling2011bayesian} as an example of the stochastic gradient Markov chain Monte Carlo scheme \cite{ma2015complete, zhang2018advances}; and (2) agnostic federated learning \cite{geyer2017differentially, mohri2019agnostic}. Our results help deliver $\mathcal O(\sqrt{\log N/N})$ high-probability generalization bounds and PAC-learnability guarantees for the two schemes. Detailed proofs are given in Section \ref{app:application}.

\subsection{Application in SGLD}
\label{sec:SGLD}

%Our theories apply to a wide spectrum of machine learning algorithms. This section introduces them to stochastic gradient Langevin dynamics (SGLD; \cite{welling2011bayesian}) as an example. %as an example of the Gaussian mechanism \cite{dwork2006calibrating, nikolov2013geometry}. Detailed proofs are given in Appendix \ref{proof:SGLD}.

%For parametric machine learning models, 

Bayesian inference aims to approximate the posterior of model parameters for parametric machine learning models and then approach the best parameter. However, the analytic expression of the posterior is inaccessible in most real-world cases. To solve this problem, Markov chain Monte Carlo (MCMC) methods are employed to infer the posterior \cite{hastings1970monte, duane1987hybrid}. In practice, MCMC can be prohibitively time-consuming on large-scale data. To address this issue, stochastic gradient Markov chain Monte Carlo (SGMCMC; \cite{ma2015complete}) introduces stochastic gradient estimate \cite{robbins1951stochastic} into Bayesian inference. A canonical example of SGMCMC algorithms is stochastic gradient Langevin dynamics (SGLD; \cite{welling2011bayesian}). SGMCMC has been applied to many areas, including topic model \cite{larochelle2012neural, zhang2020deep}, Bayesian neural network \cite{louizos2017multiplicative, roth2018bayesian, ban2019variational, ye2020optimizing}, and generative models \cite{xie2019learning, kobyzev2020normalizing}. This paper analyses SGLD as an example of the SGMCMC scheme, which is illustrated as the following chart.

{\centering
	\begin{minipage}{\linewidth}
		\begin{algorithm}[H]

  \small

			\caption{Stochastic Gradient Langevin Dynamics}%, \cite{wang2015privacy}}
			
			\label{alg:1}
			
			\begin{algorithmic}[1]
				
				\REQUIRE Samples $S=\{z_1,...z_N\}$, Gauss noise variance $\sigma$, size of mini-batch $\tau$, iteration steps $T$, learning rate $\{\eta_1,...\eta_T\}$, Regularization function $r$, Lipschitz constant $L$ of loss $l$.
				
				%\ENSURE $\{\theta_0,\theta_1,...,\theta_T\}$
				
				\STATE Initialize $\theta_0$ randomly. 
				
				\STATE For $t=1$ to $T$ do:
				
				\STATE \qquad Randomly sample a mini-batch $\mathcal{B}$ of size $\tau$;
				
				\STATE \qquad Sample $g_t$ from $\mathcal{N}(0, \sigma^2 I)$;
				
				\STATE \qquad Update \\ $\theta_t \leftarrow \theta_{t-1}-\eta_t \left[ \frac{1}{\tau} \nabla r(\theta_{t-1})+ \frac{1}{\tau} \sum_{z \in \mathcal{B}}\nabla l(z | \theta_{t-1})+g_t \right]$.
				
				%\STATE \qquad Return $\theta_{t+1}$;
				
				%\STATE \qquad Increment $t + 1 \leftarrow t$.
					
			\end{algorithmic}  
		\end{algorithm}
	\end{minipage}
}

The following theorem estimates the differential privacy and delivers a generalization bound for SGLD.

\begin{theorem}%[Differential Privacy and Generalization Bounds of SGLD]
\label{thm:SGLD}
SGLD is $(\varepsilon', \delta')$-differentially private. The factor $\varepsilon'$ is as follows,
	\begin{gather*}
	\label{eq:SGLD_epsilon}
	\varepsilon'=\sqrt{8  \log \left( \frac{1}{\tilde\delta}\right)\left(\frac{\tau^2}{N^2}T\tilde\varepsilon^2\right)} +2T \frac{\tau}{N} \tilde\varepsilon \frac{e^{2\frac{\tau}{N}\tilde\varepsilon}-1}{e^{2\frac{\tau}{N}\tilde\varepsilon}+1},
	\end{gather*}
	and the factor $\delta'$ is as follows,
	\begin{align*}
	\label{eq:SGLD_delta}
	 \delta' = & 1-\left(1-e^{2\frac{\tau}{N}\tilde\varepsilon}\frac{\frac{\tau}{N}\delta}{1+e^{2\frac{\tau}{N}\tilde\varepsilon}}\right)^{\left \lceil  \frac{N\varepsilon'}{2\tau\tilde\varepsilon}\right \rceil} \left(1-\frac{\frac{\tau}{N}\delta}{1+e^{2\frac{\tau}{N}\tilde\varepsilon}}\right)^{T-\left \lceil  \frac{N\varepsilon'}{2\tau\tilde\varepsilon}\right \rceil}%\\
	  +1-\left(1-\frac{\frac{\tau}{N}\delta}{1+e^{2\frac{\tau}{N}\tilde\varepsilon}}\right)^{T} + \tilde\delta,
	\end{align*}
	where
	\begin{equation*}
	\tilde\varepsilon = \frac{2\sqrt{2}L\sigma\frac{1}{\tau}\sqrt{\log\frac{1}{\delta}}+\frac{4}{\tau^2}L^2}{2\sigma^2},
	\end{equation*}
	and
	\begin{align*}
	\tilde\delta = & e^{-\frac{\varepsilon'+\frac{\tau}{N}T\tilde\varepsilon}{2}}\left(\frac{1}{1+e^{\frac{\tau}{N}\tilde\varepsilon}}\left(\frac{2\frac{\tau}{N}T\tilde\varepsilon}{\frac{\tau}{N}T\tilde\varepsilon-\varepsilon'}\right)\right)^T %\\
	\left(\frac{\frac{\tau}{N}\tilde\varepsilon T+\varepsilon'}{\frac{\tau}{N}T\tilde\varepsilon-\varepsilon'}\right)^{-\frac{N\varepsilon'+\tau T\tilde\varepsilon}{2\tau \tilde\varepsilon}}.
	\end{align*}
Additionally, a generalization bound is delivered by combining with Theorem \ref{thm:high_probability_privacy}. %and \ref{thm:stability_generalization_privacy_1}.
\end{theorem}

%With an assumption on convergence, 

Some existing works have also studied the privacy-preservation and generalization of SGLD. % and the wider SGMC family. 

Wang et al. \cite{wang2015privacy} proved that SGLD has "{\it privacy for free}" without injecting noise. Specifically, the authors proved that SGLD is $(\varepsilon, \delta)$-differentially private if
\begin{equation*}
	T > \frac{\varepsilon^2 N}{32 \tau \log(2/\delta)}.
\end{equation*}
%(2) SGHMC is $(\varepsilon, \delta)$-differentially private if
%\begin{equation*}
%	2(A - \hat B)/h_t \succ \frac{128 nTL^2}{\tau \varepsilon^2} \log (2/\delta) I_N;
%\end{equation*}
%and (3) SGHMC is $(\varepsilon, \delta)$-differentially private if
%\begin{equation*}
%	\frac{2a}{\eta_t} \ge \frac{128 NTL^2}{\tau \varepsilon^2} \log \left( \frac{2NT}{\tau\delta} \right) \log \left( \frac{1}{\delta} \right).
%\end{equation*}

Pensia et al. \cite{pensia2018generalization} analyzed the generalizability of SGLD %and SGHMC 
via information theory. Some works also deliver generalization bounds via algorithmic stability or the PAC-Bayesian framework \cite{hardt2016train, raginsky2017non, mou2018generalization}. %The theories help deliver an on-average generalization bound and a high-probability generalization bound for SGLD. 

Our Theorem \ref{thm:SGLD} also demonstrates that SGLD is PAC-learnable under the following assumption.

\begin{assumption}
\label{assumption:optimization}
There exist constants $K_1>0$, $K_2$ ,  $T_0$, and $N_0$, such that, for $T>T_0$ and any $N>N_0$, we have
	\begin{equation*}
	\hat{\mathcal{R}}_{S}(\mathcal{A}(S))\le \exp(-K_1 T+K_2).
	\end{equation*}
\end{assumption}

This assumption can be easily justified: the training error can almost surely achieve near-$0$ in modern machine learning. Then, we have the following remark.

\begin{remark}
%Regarding the dependence on the training sample size $N$, %the on-average generalization bound is $\mathcal O(\sqrt{\log N}/N)$ and the high-probability generalization bound is $\mathcal O(\sqrt{\log N/N})$.
Theorem \ref{thm:SGLD} implies that
\begin{equation*}
\mathbb{P}\left[\hat{\mathcal{R}}_S(\mathcal{A}(S)) \leq  O\left(\frac{T}{\sqrt{N}}\right)+\mathcal{R}(\mathcal{A}(S))\right] \ge 1-O\left(\frac{T}{\sqrt{N}}\right).
\end{equation*}
It leads to a PAC-learnable guarantee under Assumption \ref{assumption:optimization}.
\end{remark}

\subsection{Application in Agnostic Federated Learning}
\label{sec:federated_learning}

Massive amounts of personal information, including financial and medical records, have been collected. The data is highly valuable and highly sensible. This leads to a dilemma of how to extract population knowledge while protecting individual privacy. Federated learning \cite{shokri2015privacy, konevcny2016federated, mcmahan2017communication, Yang2019Federated} adapts a decentralized regime that does not access the raw data stored on personal devices. Specifically, a model is deployed on all personal devices. The central central server collects the gradients calculated on the personal devices and then distributes weight updates. This mechanism sheds light on solving the privacy-preserving problem. The following algorithm designed by \cite{geyer2017differentially, mohri2019agnostic} further enhances the privacy preservation to protect client identity from differential attacks.

{\centering
  
  \begin{minipage}{\linewidth}
   \begin{algorithm}[H]
    
    \small
    \caption{Differentially Private Federated Learning}%, \cite{geyer2017differentially}}
    
    \label{alg:2}
    
    \begin{algorithmic}[1]
     
     \REQUIRE Clients \{$c_1,...c_N$\}, Gaussian noise variance $\sigma$, size of mini-batch $\tau$, iteration steps $T$, upper bound $L$ of the step size.
     
     \STATE Initialize $\theta_0$ randomly. 
     
     \STATE For $t=1$ to $T$ do:
     
     \STATE \qquad Randomly sample a mini-batch of clients of size $\tau$;
     
     \STATE \qquad Randomly sample $b_t$ from $\mathcal{N}(0, L^2\sigma^2 I)$;
     
     \STATE \qquad Central curator distributes $\theta_{t-1}$ to the clients in the mini-batch $\mathcal{B}$;
     
     \STATE \qquad Update $\theta_{t+1} \leftarrow \theta_{t}+ \left(\frac{1}{B} \sum_{i \in \mathcal{B}} \frac{\text{ClientUpdate}(c_i,\theta_{t})}{\max \left(1,\frac{\Vert h_i\Vert_2}{L}\right)} + b_t\right)$.
          
    \end{algorithmic}  
   \end{algorithm}
  \end{minipage}
 }

The following theorem estimates the differential privacy and delivers a generalization bound for agnostic federated learning.

\begin{theorem}[Differential Privacy and Generalization Bounds of Differentially Private Federated Learning]
\label{thm:federated}
Agnostic federated learning is $(\varepsilon', \delta')$-differentially private. The factor $\varepsilon'$ is as follows,
	\begin{gather}
	\label{eq:federated_epsilon}
	\varepsilon'=\sqrt{8  \log \left( \frac{1}{\tilde\delta}\right)\left(\frac{\tau^2}{N^2}T\varepsilon^2\right)} +2T \frac{\tau}{N} \varepsilon \frac{e^{2\frac{\tau}{N}\varepsilon}-1}{e^{2\frac{\tau}{N}\varepsilon}+1},
	\end{gather}
	and the factor $\delta'$ is defined as follows,
	\begin{align*}
	 \delta' = & 1-\left(1-e^{2\frac{\tau}{N}\tilde\varepsilon}\frac{\frac{\tau}{N}\delta}{1+e^{2\frac{\tau}{N}\tilde\varepsilon}}\right)^{\left \lceil  \frac{N\varepsilon'}{2\tau\tilde\varepsilon}\right \rceil}\left(1-\frac{\frac{\tau}{N}\delta}{1+e^{2\frac{\tau}{N}\tilde\varepsilon}}\right)^{T-\left \lceil  \frac{N\varepsilon'}{2\tau\tilde\varepsilon}\right \rceil} %\\
	 +1-\left(1-\frac{\frac{\tau}{N}\delta}{1+e^{2\frac{\tau}{N}\tilde\varepsilon}}\right)^{T} + \tilde\delta,
	\end{align*}
	where
	\begin{equation*}
	\tilde\varepsilon = \frac{4\sigma\frac{1}{\tau}\sqrt{\log\frac{1}{\delta}}+\frac{1}{\tau^2}}{2\sigma^2},
	\end{equation*}
	and
	\begin{align*}
	\tilde\delta = & e^{-\frac{\varepsilon'+\frac{\tau}{N}T\tilde\varepsilon}{2}}\left(\frac{1}{1+e^{\frac{\tau}{N}\tilde\varepsilon}}\left(\frac{2\frac{\tau}{N}T\tilde\varepsilon}{\frac{\tau}{N}T\tilde\varepsilon-\varepsilon'}\right)\right)^T %\\
	\left(\frac{\frac{\tau}{N}\tilde\varepsilon T+\varepsilon'}{\frac{\tau}{N}T\tilde\varepsilon-\varepsilon'}\right)^{-\frac{N\varepsilon'+\tau T\tilde\varepsilon}{2\tau \tilde\varepsilon}}.
	\end{align*}
Additionally, a generalization bound is delivered by combining with Theorem \ref{thm:high_probability_privacy}. %and \ref{thm:stability_generalization_privacy_1}.
\end{theorem}

%Additionally, we have t
The following remark gives a PAC-learnable guarantee for agnostic federated learning. 

\begin{remark}
Theorem \ref{thm:federated} implies that
\begin{equation*}
\mathbb{P}\left[\hat{\mathcal{R}}_S(\mathcal{A}(S)) \leq  O\left(\frac{T}{\sqrt{N}}\right)+\mathcal{R}(\mathcal{A}(S))\right] \ge 1-O\left(\frac{T}{\sqrt{N}}\right).
\end{equation*}
It leads to a PAC-learnable guarantee under Assumption \ref{assumption:optimization}.
\end{remark}

%With an assumption on convergence, 
%It secures that agnostic federated learning is PAC-learnable under the aforementioned Assumption \ref{assumption:optimization}.

\section{Conclusion}
\label{sec:conclusion}

This paper studies the relationships between generalization and privacy preservation in two steps. We first establish the relationship between generalization and privacy preservation for any machine learning algorithm. Specifically, we prove a high-probability bound for differentially private learning algorithms based on a novel on-average generalization bound for multi-database algorithms. This high-probability generalization bound delivers a PAC-learnable guarantee for differentially private learning algorithms. Then, we prove three composition theorems that calculate the $(\varepsilon', \delta')$-differential privacy of an iterative algorithm. By integrating the two steps, we establish generalization guarantees for iterative differentially private machine learning algorithms. Compared with existing works, our theoretical results are strictly tighter and apply to a wider application domain. We then use them to study the privacy preservation and further the generalization of stochastic gradient Langevin dynamics (SGLD), as an example of the stochastic gradient Markov chain Monte Carlo, and agnostic federated learning. We obtain the approximation of differential privacy of SGLD and agnostic federated learning which further leads to high-probability generalization bounds that do not explicitly rely on the model size which would be prohibitively large in many deep models.

\bibliographystyle{plain}
\bibliography{bridging_the_gap}

\begin{thebibliography}{10}

\bibitem{abadi2016deep}
Martin Abadi, Andy Chu, Ian Goodfellow, H~Brendan McMahan, Ilya Mironov, Kunal
  Talwar, and Li~Zhang.
\newblock Deep learning with differential privacy.
\newblock In {\em Proceedings of ACM SIGSAC Conference on Computer and
  Communications Security}, pages 308--318, 2016.

\bibitem{arachchige2019local}
Pathum Chamikara~Mahawaga Arachchige, Peter Bertok, Ibrahim Khalil, Dongxi Liu,
  Seyit Camtepe, and Mohammed Atiquzzaman.
\newblock Local differential privacy for deep learning.
\newblock {\em IEEE Internet of Things Journal}, 2019.

\bibitem{ban2019variational}
Yutong Ban, Xavier Alameda-Pineda, Laurent Girin, and Radu Horaud.
\newblock Variational bayesian inference for audio-visual tracking of multiple
  speakers.
\newblock {\em IEEE Transactions on Pattern Analysis and Machine Intelligence},
  2019.

\bibitem{bartlett2017spectrally}
Peter~L Bartlett, Dylan~J Foster, and Matus~J Telgarsky.
\newblock Spectrally-normalized margin bounds for neural networks.
\newblock In {\em Advances in Neural Information Processing Systems}, pages
  6240--6249, 2017.

\bibitem{JMLR:v20:17-612}
Peter~L. Bartlett, Nick Harvey, Christopher Liaw, and Abbas Mehrabian.
\newblock Nearly-tight vc-dimension and pseudodimension bounds for piecewise
  linear neural networks.
\newblock {\em The Journal of Machine Learning Research}, 20(63):1--17, 2019.

\bibitem{bartlett2002rademacher}
Peter~L Bartlett and Shahar Mendelson.
\newblock Rademacher and gaussian complexities: Risk bounds and structural
  results.
\newblock {\em The Journal of Machine Learning Research}, 3(Nov):463--482,
  2002.

\bibitem{beimel2010bounds}
Amos Beimel, Shiva~Prasad Kasiviswanathan, and Kobbi Nissim.
\newblock Bounds on the sample complexity for private learning and private data
  release.
\newblock In {\em Proceedings of Theory of Cryptography Conference}, pages
  437--454, 2010.

\bibitem{blumer1989learnability}
Anselm Blumer, Andrzej Ehrenfeucht, David Haussler, and Manfred~K Warmuth.
\newblock Learnability and the vapnik-chervonenkis dimension.
\newblock {\em Journal of the ACM}, 36(4):929--965, 1989.

\bibitem{boucheron2013concentration}
St{\'e}phane Boucheron, G{\'a}bor Lugosi, and Pascal Massart.
\newblock {\em Concentration inequalities: A nonasymptotic theory of
  independence}.
\newblock Oxford University Press, 2013.

\bibitem{bousquet2002stability}
Olivier Bousquet and Andr{\'e} Elisseeff.
\newblock Stability and generalization.
\newblock {\em The Journal of Machine Learning Research}, 2(Mar):499--526,
  2002.

\bibitem{bun2016concentrated}
Mark Bun and Thomas Steinke.
\newblock Concentrated differential privacy: Simplifications, extensions, and
  lower bounds.
\newblock In {\em Proceedings of Theory of Cryptography Conference}, pages
  635--658, 2016.

\bibitem{chaudhuri2019capacity}
Kamalika Chaudhuri, Jacob Imola, and Ashwin Machanavajjhala.
\newblock Capacity bounded differential privacy.
\newblock {\em arXiv preprint arXiv:1907.02159}, 2019.

\bibitem{cuff2016differential}
Paul Cuff and Lanqing Yu.
\newblock Differential privacy as a mutual information constraint.
\newblock In {\em Proceedings of ACM SIGSAC Conference on Computer and
  Communications Security}, pages 43--54, 2016.

\bibitem{duane1987hybrid}
Simon Duane, Anthony~D Kennedy, Brian~J Pendleton, and Duncan Roweth.
\newblock Hybrid monte carlo.
\newblock {\em Physics Letters B}, 195(2):216--222, 1987.

\bibitem{dudley1967sizes}
Richard~M Dudley.
\newblock The sizes of compact subsets of hilbert space and continuity of
  gaussian processes.
\newblock {\em Journal of Functional Analysis}, 1(3):290--330, 1967.

\bibitem{10.1007/11787006_1}
Cynthia Dwork.
\newblock Differential privacy.
\newblock In Michele Bugliesi, Bart Preneel, Vladimiro Sassone, and Ingo
  Wegener, editors, {\em Automata, Languages and Programming}, pages 1--12,
  Berlin, Heidelberg, 2006. Springer Berlin Heidelberg.

\bibitem{dwork2015preserving}
Cynthia Dwork, Vitaly Feldman, Moritz Hardt, Toniann Pitassi, Omer Reingold,
  and Aaron~Leon Roth.
\newblock Preserving statistical validity in adaptive data analysis.
\newblock In {\em Proceedings of Annual ACM Symposium on Theory of Computing},
  pages 117--126, 2015.

\bibitem{dwork2014algorithmic}
Cynthia Dwork and Aaron Roth.
\newblock The algorithmic foundations of differential privacy.
\newblock {\em Foundations and Trends{\textregistered} in Theoretical Computer
  Science}, 9(3--4):211--407, 2014.

\bibitem{dwork2016concentrated}
Cynthia Dwork and Guy~N Rothblum.
\newblock Concentrated differential privacy.
\newblock {\em arXiv preprint arXiv:1603.01887}, 2016.

\bibitem{dwork2010boosting}
Cynthia Dwork, Guy~N Rothblum, and Salil Vadhan.
\newblock Boosting and differential privacy.
\newblock In {\em Proceedings of IEEE Annual Symposium on Foundations of
  Computer Science}, pages 51--60, 2010.

\bibitem{fischer2018deep}
Thomas Fischer and Christopher Krauss.
\newblock Deep learning with long short-term memory networks for financial
  market predictions.
\newblock {\em European Journal of Operational Research}, 270(2):654--669,
  2018.

\bibitem{geumlek2017renyi}
Joseph Geumlek, Shuang Song, and Kamalika Chaudhuri.
\newblock Renyi differential privacy mechanisms for posterior sampling.
\newblock In {\em Advances in Neural Information Processing Systems}, pages
  5289--5298, 2017.

\bibitem{geyer2017differentially}
Robin~C Geyer, Tassilo Klein, and Moin Nabi.
\newblock Differentially private federated learning: A client level
  perspective.
\newblock In {\em Advances in Neural Information Processing Systems}, 2017.

\bibitem{golowich2017size}
Noah Golowich, Alexander Rakhlin, and Ohad Shamir.
\newblock Size-independent sample complexity of neural networks.
\newblock In {\em Proceedings of Annual Conference on Learning Theory}, pages
  297--299, 2018.

\bibitem{hardt2016train}
Moritz Hardt, Ben Recht, and Yoram Singer.
\newblock Train faster, generalize better: Stability of stochastic gradient
  descent.
\newblock In {\em Proceedings of International Conference on Machine Learning},
  pages 1225--1234, 2016.

\bibitem{hastings1970monte}
W~Keith Hastings.
\newblock Monte carlo sampling methods using markov chains and their
  applications.
\newblock 1970.

\bibitem{haussler1995sphere}
David Haussler.
\newblock Sphere packing numbers for subsets of the boolean n-cube with bounded
  vapnik-chervonenkis dimension.
\newblock {\em Journal of Combinatorial Theory, Series A}, 69(2):217--232,
  1995.

\bibitem{he2019control}
Fengxiang He, Tongliang Liu, and Dacheng Tao.
\newblock Control batch size and learning rate to generalize well: Theoretical
  and empirical evidence.
\newblock In {\em Advances in Neural Information Processing Systems}, pages
  1143--1152, 2019.

\bibitem{he2020piecewise}
Fengxiang He, Bohan Wang, and Dacheng Tao.
\newblock Piecewise linear activations substantially shape the loss surfaces of
  neural networks.
\newblock In {\em International Conference on Learning Representations}, 2020.

\bibitem{kairouz2015composition}
Peter Kairouz, Sewoong Oh, and Pramod Viswanath.
\newblock The composition theorem for differential privacy.
\newblock In {\em International conference on machine learning}, pages
  1376--1385, 2015.

\bibitem{kairouz2017composition}
Peter Kairouz, Sewoong Oh, and Pramod Viswanath.
\newblock The composition theorem for differential privacy.
\newblock {\em IEEE Transactions on Information Theory}, 63(6):4037--4049,
  2017.

\bibitem{kawaguchi2016deep}
Kenji Kawaguchi.
\newblock Deep learning without poor local minima.
\newblock In {\em Advances in neural information processing systems}, pages
  586--594, 2016.

\bibitem{kobyzev2020normalizing}
Ivan Kobyzev, Simon~J.D. Prince, and Marcus~A. Brubaker.
\newblock Normalizing flows: An introduction and review of current methods.
\newblock {\em IEEE Transactions on Pattern Analysis and Machine Intelligence},
  2020.

\bibitem{koltchinskii2001rademacher}
Vladimir Koltchinskii.
\newblock Rademacher penalties and structural risk minimization.
\newblock {\em IEEE Transactions on Information Theory}, 47(5):1902--1914,
  2001.

\bibitem{koltchinskii2000rademacher}
Vladimir Koltchinskii and Dmitriy Panchenko.
\newblock Rademacher processes and bounding the risk of function learning.
\newblock In {\em High Dimensional Probability II}, pages 443--457. Springer,
  2000.

\bibitem{konevcny2016federated}
Jakub Kone{\v{c}}n{\`y}, H~Brendan McMahan, Felix~X Yu, Peter Richt{\'a}rik,
  Ananda~Theertha Suresh, and Dave Bacon.
\newblock Federated learning: Strategies for improving communication
  efficiency.
\newblock In {\em Advances in Neural Information Processing Systems Workshop on
  Private Multi-Party Machine Learning}, 2016.

\bibitem{kullback1951information}
Solomon Kullback and Richard~A Leibler.
\newblock On information and sufficiency.
\newblock {\em The Annals of Mathematical Statistics}, 22(1):79--86, 1951.

\bibitem{larochelle2012neural}
Hugo Larochelle and Stanislas Lauly.
\newblock A neural autoregressive topic model.
\newblock In {\em Advances in Neural Information Processing Systems}, pages
  2708--2716, 2012.

\bibitem{liang2019fisher}
Tengyuan Liang, Tomaso Poggio, Alexander Rakhlin, and James Stokes.
\newblock Fisher-rao metric, geometry, and complexity of neural networks.
\newblock In {\em Proceedings of International Conference on Artificial
  Intelligence and Statistics}, pages 888--896, 2019.

\bibitem{liao2017hypothesis}
Jiachun Liao, Lalitha Sankar, Vincent~YF Tan, and Flavio du~Pin~Calmon.
\newblock Hypothesis testing under mutual information privacy constraints in
  the high privacy regime.
\newblock {\em IEEE Transactions on Information Forensics and Security},
  13(4):1058--1071, 2017.

\bibitem{louizos2017multiplicative}
Christos Louizos and Max Welling.
\newblock Multiplicative normalizing flows for variational bayesian neural
  networks.
\newblock In {\em Proceedings of International Conference on Machine Learning},
  pages 2218--2227, 2017.

\bibitem{ma2015complete}
Yi-An Ma, Tianqi Chen, and Emily Fox.
\newblock A complete recipe for stochastic gradient mcmc.
\newblock In {\em Advances in Neural Information Processing Systems}, pages
  2917--2925, 2015.

\bibitem{mandt2017stochastic}
Stephan Mandt, Matthew~D Hoffman, and David~M Blei.
\newblock Stochastic gradient descent as approximate bayesian inference.
\newblock {\em The Journal of Machine Learning Research}, 18(1):4873--4907,
  2017.

\bibitem{mcallester1999some}
David~A McAllester.
\newblock Some pac-bayesian theorems.
\newblock {\em Machine Learning}, 37(3):355--363, 1999.

\bibitem{mcmahan2017communication}
H~Brendan McMahan, Eider Moore, Daniel Ramage, Seth Hampson, and Blaise
  Aguera~y Arcas.
\newblock Communication-efficient learning of deep networks from decentralized
  data.
\newblock In {\em Proceedings of International Conference on Artificial
  Intelligence and Statistics}, 2017.

\bibitem{mcsherry2007mechanism}
Frank McSherry and Kunal Talwar.
\newblock Mechanism design via differential privacy.
\newblock In {\em IEEE Symposium on Foundations of Computer Science}, pages
  94--103, 2007.

\bibitem{mironov2017renyi}
Ilya Mironov.
\newblock R{\'e}nyi differential privacy.
\newblock In {\em Proceedings of IEEE Computer Security Foundations Symposium},
  pages 263--275, 2017.

\bibitem{mohri2018foundations}
Mehryar Mohri, Afshin Rostamizadeh, and Ameet Talwalkar.
\newblock {\em Foundations of machine learning}.
\newblock MIT Press, 2018.

\bibitem{mohri2019agnostic}
Mehryar Mohri, Gary Sivek, and Ananda~Theertha Suresh.
\newblock Agnostic federated learning.
\newblock In {\em Proceedings of International Conference on Machine Learning},
  pages 4615--4625, 2019.

\bibitem{mou2018generalization}
Wenlong Mou, Liwei Wang, Xiyu Zhai, and Kai Zheng.
\newblock Generalization bounds of sgld for non-convex learning: Two
  theoretical viewpoints.
\newblock In {\em Proceedings of Annual Conference On Learning Theory}, pages
  605--638, 2018.

\bibitem{musavi1994generalization}
Mohamad~T Musavi, Khue~Hiang Chan, Donald~M Hummels, and K~Kalantri.
\newblock On the generalization ability of neural network classifiers.
\newblock {\em IEEE Transactions on Pattern Analysis and Machine Intelligence},
  16(6):659--663, 1994.

\bibitem{neyshabur2017pac}
Behnam Neyshabur, Srinadh Bhojanapalli, David McAllester, and Nathan Srebro.
\newblock A pac-bayesian approach to spectrally-normalized margin bounds for
  neural networks.
\newblock In {\em International Conference on Learning Representations}, 2018.

\bibitem{nissim2015generalization}
Kobbi Nissim and Uri Stemmer.
\newblock On the generalization properties of differential privacy.
\newblock {\em CoRR, abs/1504.05800}, 2015.

\bibitem{oneto2017differential}
Luca Oneto, Sandro Ridella, and Davide Anguita.
\newblock Differential privacy and generalization: Sharper bounds with
  applications.
\newblock {\em Pattern Recognition Letters}, 89:31--38, 2017.

\bibitem{pensia2018generalization}
Ankit Pensia, Varun Jog, and Po-Ling Loh.
\newblock Generalization error bounds for noisy, iterative algorithms.
\newblock In {\em Proceedings of IEEE International Symposium on Information
  Theory}, pages 546--550, 2018.

\bibitem{pittaluga2016pre}
Francesco Pittaluga and Sanjeev~Jagannatha Koppal.
\newblock Pre-capture privacy for small vision sensors.
\newblock {\em IEEE Transactions on Pattern Analysis and Machine Intelligence},
  39(11):2215--2226, 2016.

\bibitem{raginsky2017non}
Maxim Raginsky, Alexander Rakhlin, and Matus Telgarsky.
\newblock Non-convex learning via stochastic gradient langevin dynamics: a
  nonasymptotic analysis.
\newblock In {\em Proceedings of Annual Conference on Learning Theory}, pages
  1674--1703, 2017.

\bibitem{robbins1951stochastic}
Herbert Robbins and Sutton Monro.
\newblock A stochastic approximation method.
\newblock {\em The Annals of Mathematical Statistics}, pages 400--407, 1951.

\bibitem{rogers1978finite}
William~H Rogers and Terry~J Wagner.
\newblock A finite sample distribution-free performance bound for local
  discrimination rules.
\newblock {\em The Annals of Statistics}, pages 506--514, 1978.

\bibitem{roth2018bayesian}
Wolfgang Roth and Franz Pernkopf.
\newblock Bayesian neural networks with weight sharing using dirichlet
  processes.
\newblock {\em IEEE Transactions on Pattern Analysis and Machine Intelligence},
  42(1):246--252, 2018.

\bibitem{shokri2015privacy}
Reza Shokri and Vitaly Shmatikov.
\newblock Privacy-preserving deep learning.
\newblock In {\em Proceedings of ACM SIGSAC Conference on Computer and
  Communications Security}, pages 1310--1321, 2015.

\bibitem{snelick2005large}
Robert Snelick, Umut Uludag, Alan Mink, Mike Indovina, and Anil Jain.
\newblock Large-scale evaluation of multimodal biometric authentication using
  state-of-the-art systems.
\newblock {\em IEEE transactions on pattern analysis and machine intelligence},
  27(3):450--455, 2005.

\bibitem{tu2020understanding}
Zhuozhuo Tu, Fengxiang He, and Dacheng Tao.
\newblock Understanding generalization in recurrent neural networks.
\newblock In {\em International Conference on Learning Representations}, 2020.

\bibitem{vapnik2006estimation}
Vladimir Vapnik.
\newblock {\em Estimation of dependences based on empirical data}.
\newblock Springer Science \& Business Media, 2006.

\bibitem{vapnik2013nature}
Vladimir Vapnik.
\newblock {\em The nature of statistical learning theory}.
\newblock Springer Science \& Business Media, 2013.

\bibitem{wang2012framework}
Fei Wang, Noah Lee, Jianying Hu, Jimeng Sun, Shahram Ebadollahi, and Andrew~F
  Laine.
\newblock A framework for mining signatures from event sequences and its
  applications in healthcare data.
\newblock {\em IEEE transactions on pattern analysis and machine intelligence},
  35(2):272--285, 2012.

\bibitem{wang2016relation}
Weina Wang, Lei Ying, and Junshan Zhang.
\newblock On the relation between identifiability, differential privacy, and
  mutual-information privacy.
\newblock {\em IEEE Transactions on Information Theory}, 62(9):5018--5029,
  2016.

\bibitem{wang2015privacy}
Yu-Xiang Wang, Stephen Fienberg, and Alex Smola.
\newblock Privacy for free: Posterior sampling and stochastic gradient monte
  carlo.
\newblock In {\em Proceedings of International Conference on Machine Learning},
  pages 2493--2502, 2015.

\bibitem{welling2011bayesian}
Max Welling and Yee~W Teh.
\newblock Bayesian learning via stochastic gradient langevin dynamics.
\newblock In {\em Proceedings of International Conference on Machine Learning},
  pages 681--688, 2011.

\bibitem{xie2019learning}
Jianwen Xie, Song-Chun Zhu, and Ying~Nian Wu.
\newblock Learning energy-based spatial-temporal generative convnets for
  dynamic patterns.
\newblock {\em IEEE Transactions on Pattern Analysis and Machine Intelligence},
  2019.

\bibitem{xu2011sparse}
Huan Xu, Constantine Caramanis, and Shie Mannor.
\newblock Sparse algorithms are not stable: A no-free-lunch theorem.
\newblock {\em IEEE Transactions on Pattern Analysis and Machine Intelligence},
  34(1):187--193, 2011.

\bibitem{Yang2019Federated}
Qiang Yang, Yang Liu, Tianjian Chen, and Yongxin Tong.
\newblock Federated machine learning: concept and applications.
\newblock {\em ACM Transactions on Intelligent Systems and Technology},
  10(2):12:1--12:19, 2019.

\bibitem{ye2020optimizing}
Qiaoling Ye, Arash~A. Amini, and Qing Zhou.
\newblock Optimizing regularized cholesky score for order-based learning of
  bayesian networks.
\newblock {\em IEEE Transactions on Pattern Analysis and Machine Intelligence},
  2020.

\bibitem{yun2019small}
Chulhee Yun, Suvrit Sra, and Ali Jadbabaie.
\newblock Small nonlinearities in activation functions create bad local minima
  in neural networks.
\newblock In {\em International Conference on Learning Representations}, 2019.

\bibitem{zhang2018advances}
Cheng Zhang, Judith B{\"u}tepage, Hedvig Kjellstr{\"o}m, and Stephan Mandt.
\newblock Advances in variational inference.
\newblock {\em IEEE Transactions on Pattern Analysis and Machine Intelligence},
  41(8):2008--2026, 2018.

\bibitem{zhang2020deep}
Hao Zhang, Bo~Chen, Yulai Cong, Dandan Guo, Hongwei Liu, and Mingyuan Zhou.
\newblock Deep autoencoding topic model with scalable hybrid bayesian
  inference.
\newblock {\em IEEE Transactions on Pattern Analysis and Machine Intelligence},
  2020.

\end{thebibliography}

\appendix

\section{Proofs of Generalization Bounds via Differential Privacy}
\label{proof:privacy-generalization}

This appendix collects all the proofs of  the generalization bounds. %For brevity, we only prove one side for all the bounds, since the proof of the other side follows exactly the same routine.
It is organized as follows: (1) Appendix \ref{app:multi_database} proves Theorem \ref{thm:multi_database}; and (2) Appendix \ref{proof:generalization} proves Theorem \ref{thm:high_probability_privacy}.

\subsection{Proof of Theorem \ref{thm:multi_database}}
\label{app:multi_database}

	\begin{proof}[Proof of Theorem \ref{thm:multi_database}]
		The left side of eq. (\ref{eq:stab_databases}) can be rewritten as 
		\begin{align*}
		&\underset{\vec{S} \sim \mathcal{D}^{kN}}{\mathbb{E}}\left[\underset{ \mathcal{A}(\vec{S})}{\mathbb{E}}\left[\hat{\mathcal{R}}_{S_{i_{\mathcal{A}(\vec{S})}}}\left(h_{\mathcal{A}(\vec{S})}\right)\right]\right]=	\underset{\vec{S} \sim \mathcal{D}^{kN}}{\mathbb{E}}\left[\underset{ \mathcal{A}(\vec{S})}{\mathbb{E}}\left[\mathbb{E}_{z\sim S_{i_{\mathcal{A}(\vec{S})}} }\left[l\left(h_{\mathcal{A}(\vec{S})},z\right)\right]\right]\right]
		\\
		\overset{(*)}{=}&\underset{\vec{S} \sim \mathcal{D}^{kN}}{\mathbb{E}}\left[\underset{ \mathcal{A}(\vec{S})}{\mathbb{E}}\left[\underset{\vec{z}\sim \vec{S}}{\mathbb{E}} \left[l\left(h_{\mathcal{A}(\vec{S})},z_{i_{\mathcal{A}(\vec{S})}}\right)\right]\right]\right]
		=\underset{\vec{S} \sim \mathcal{D}^{kN}}{\mathbb{E}}\left[\underset{\vec{z}\sim \vec{S}}{\mathbb{E}}\left[\underset{ \mathcal{A}(\vec{S})}{\mathbb{E}} \left[l\left(h_{\mathcal{A}(\vec{S})},z_{i_{\mathcal{A}(\vec{S})}}\right)\right]\right]\right]
		\\
		=&\underset{\vec{S} \sim \mathcal{D}^{kN}}{\mathbb{E}}\left[\underset{\vec{z}\sim \vec{S}}{\mathbb{E}}\left[\underset{ \mathcal{A}(\vec{S})}{\mathbb{E}} \left[l\left(h_{\mathcal{A}(\vec{S})},z_{i_{\mathcal{A}(\vec{S})}}\right)\right]\right]\right]
		\\
		=&\underset{\vec{S} \sim \mathcal{D}^{kN}}{\mathbb{E}}\left[\underset{\vec{z}\sim \vec{S}}{\mathbb{E}}\left[\int_{0}^{1}\mathbb{P}\left(l\left(h_{\mathcal{A}(\vec{S})},z_{i_{\mathcal{A}(\vec{S})}}\right)\le t\right) dt\right]\right]
		\\
		=&\underset{\vec{S} \sim \mathcal{D}^{kN}}{\mathbb{E}}\left[\underset{\vec{z}\sim \vec{S}}{\mathbb{E}}\left[\sum_{i=1}^{k}\int_{0}^{1}\mathbb{P}\left(l\left(h_{\mathcal{A}(\vec{S})},z_{i}\right)\le t, i_{\mathcal{A}(\vec{S})}=i\right) \text{d}t\right]\right]	,	
		\end{align*}
		where $\vec{z}$ in the right side of $(*)$ is defined as $\{z_1,\cdots,z_k\}$,  $z_i$ is uniformly selected from $S_i$.
		Since $\mathcal{A}$ is $(\varepsilon,\delta)$-differentially private, we further have
		\begin{align*}
		&\underset{\vec{S} \sim \mathcal{D}^{kN}}{\mathbb{E}}\left[\underset{\vec{z}\sim \vec{S}}{\mathbb{E}}\left[\sum_{i=1}^{k}\int_{0}^{1}\mathbb{P}\left(l\left(h_{\mathcal{A}(\vec{S})},z_{i}\right)\le t, i_{\mathcal{A}(\vec{S})}=i\right) \text{d}t\right]\right]
		\\
		\le&\underset{\vec{S} \sim \mathcal{D}^{kN}}{\mathbb{E}}\left[\underset{\vec{z}\sim \vec{S},z_0\sim D}{\mathbb{E}}\left[\sum_{i=1}^{k}\int_{0}^{1}e^\varepsilon\mathbb{P}\left(l\left(h_{\mathcal{A}(\vec{S}^{z_i:z_0})},z_{i}\right)\le t, i_{\mathcal{A}(\vec{S}^{z_i:z_0})}=i\right) +\delta \text{ d}t\right]\right]
		\\
		=&e^\varepsilon\underset{\vec{S} \sim \mathcal{D}^{kN}}{\mathbb{E}}\left[\underset{\vec{z}\sim \vec{S},z_0\sim D}{\mathbb{E}}\left[\sum_{i=1}^{k}\int_{0}^{1}\mathbb{P}\left(l\left(h_{\mathcal{A}(\vec{S}^{z_i:z_0})},z_{i}\right)\le t, i_{\mathcal{A}(\vec{S}^{z_i:z_0})}=i\right)  \text{ d}t\right]\right]+k\delta
		\\
		=&\sum_{i=1}^{k}e^\varepsilon\underset{\vec{S} \sim \mathcal{D}^{kN}}{\mathbb{E}}\left[\underset{\vec{z}\sim \vec{S},z_0\sim D}{\mathbb{E}}\left[\int_{0}^{1}\mathbb{P}\left(l\left(h_{\mathcal{A}(\vec{S}^{z_i:z_0})},z_{i}\right)\le t, i_{\mathcal{A}(\vec{S}^{z_i:z_0})}=i\right)  \text{ d}t\right]\right]+k\delta
		\\
		=&\sum_{i=1}^{k}e^\varepsilon\underset{\vec{S}' \sim \mathcal{D}^{kN-1}}{\mathbb{E}}\left[\underset{z_i\sim D,z_0\sim D}{\mathbb{E}}\left[\int_{0}^{1}\mathbb{P}\left(l\left(h_{\mathcal{A}(\vec{S}'\cup\{z_0\})},z_{i}\right)\le t, i_{\mathcal{A}(\vec{S}'\cup\{z_0\})}=i\right)  \text{ d}t\right]\right]+k\delta.
		\end{align*}
		Let $\vec{S}=\vec{S}'\cup\{z_0\}$ and $z=z_i$ (it is without less of generality since all $z_i$ is i.i.d. drawn from $\mathcal D$). Since $\vec{S}'\cup\{z_0\} \sim \mathcal{D}^{kN}$, we have
		\begin{align*}
		&\sum_{i=1}^{k}e^\varepsilon\underset{\vec{S}' \sim \mathcal{D}^{kN-1}}{\mathbb{E}}\left[\underset{z_i\sim D,z_0\sim D}{\mathbb{E}}\left[\int_{0}^{1}\mathbb{P}\left(l\left(h_{\mathcal{A}(\vec{S}'\cup\{z_0\})},z_{i}\right)\le t, i_{\mathcal{A}(\vec{S}'\cup\{z_0\})}=i\right)  \text{ d}t\right]\right]+k\delta
		\\
		=&\sum_{i=1}^{k}e^\varepsilon\underset{\vec{S} \sim \mathcal{D}^{kN}}{\mathbb{E}}\left[\underset{z\sim \mathcal{D}}{\mathbb{E}}\left[\int_{0}^{1}\mathbb{P}\left(l\left(h_{\mathcal{A}(\vec{S})},z\right)\le t, i_{\mathcal{A}(\vec{S})}=i\right)  \text{ d}t\right]\right]+k\delta
		\\
		=&e^\varepsilon\underset{\vec{S} \sim \mathcal{D}^{kN}}{\mathbb{E}}\left[\underset{z\sim \mathcal{D}}{\mathbb{E}}\left[\int_{0}^{1}\mathbb{P}\left(l\left(h_{\mathcal{A}(\vec{S})},z\right)\le t\right)  \text{ d}t\right]\right]+k\delta
		\\
		=&
		e^\varepsilon\underset{\vec{S} \sim \mathcal{D}^{kN}}{\mathbb{E}}\left[\underset{z\sim \mathcal{D}}{\mathbb{E}}\left[\mathbb{E}_{\mathcal{A}(S)}\left[l\left(h_{\mathcal{A}(\vec{S})},z\right)\right] \right]\right]+k\delta.
		\end{align*}
		Therefore, we have 
		\begin{equation}
		\label{eq:stab_middle}
		\underset{\vec{S} \sim \mathcal{D}^{kN}}{\mathbb{E}}\left[\underset{ \mathcal{A}(\vec{S})}{\mathbb{E}}\left[\hat{\mathcal{R}}_{S_{i_{\mathcal{A}(\vec{S})}}}(h_{\mathcal{A}(\vec{S})})\right]\right] \leq k\delta+e^\varepsilon\underset{\vec{S} \sim \mathcal{D}^{kN}}{\mathbb{E}}[\underset{ \mathcal{A}(\vec{S})}{\mathbb{E}}[\mathcal{R}\left(h_{\mathcal{A}(\vec{S})}\right)]].
		\end{equation}
		Rearranging eq. (\ref{eq:stab_middle}), we have
		\begin{align*}
		e^{-\varepsilon}\underset{\vec{S} \sim \mathcal{D}^{kN}}{\mathbb{E}}\left[\underset{ \mathcal{A}(\vec{S})}{\mathbb{E}}\left[\hat{\mathcal{R}}_{S_{i_{\mathcal{A}(\vec{S})}}}(h_{\mathcal{A}(\vec{S})})\right]\right] &\leq e^{-\varepsilon}k\delta+\underset{\vec{S} \sim \mathcal{D}^{kN}}{\mathbb{E}}\left[\underset{ \mathcal{A}(\vec{S})}{\mathbb{E}}\left[\mathcal{R}\left(h_{\mathcal{A}(\vec{S})}\right)\right]\right]
		\\
		-\underset{\vec{S} \sim \mathcal{D}^{kN}}{\mathbb{E}}[\underset{ \mathcal{A}(\vec{S})}{\mathbb{E}}[\mathcal{R}\left(h_{\mathcal{A}(\vec{S})}\right)]]&\leq e^{-\varepsilon}k\delta-e^{-\varepsilon}\underset{\vec{S} \sim \mathcal{D}^{kN}}{\mathbb{E}}\left[\underset{ \mathcal{A}(\vec{S})}{\mathbb{E}}\left[\hat{\mathcal{R}}_{S_{i_{\mathcal{A}(\vec{S})}}}(h_{\mathcal{A}(\vec{S})})\right]\right],
		\end{align*}
		which leads to
		\begin{align*}
		&\underset{\vec{S} \sim \mathcal{D}^{kN}}{\mathbb{E}}\left[\underset{ \mathcal{A}(\vec{S})}{\mathbb{E}}\left[\hat{\mathcal{R}}_{S_{i_{\mathcal{A}(\vec{S})}}}(h_{\mathcal{A}(\vec{S})})\right]\right]-\underset{\vec{S} \sim \mathcal{D}^{kN}}{\mathbb{E}}[\underset{ \mathcal{A}(\vec{S})}{\mathbb{E}}[\mathcal{R}\left(h_{\mathcal{A}(\vec{S})}\right)]]
		\\
		\leq &e^{-\varepsilon}k\delta-e^{-\varepsilon}\underset{\vec{S} \sim \mathcal{D}^{kN}}{\mathbb{E}}\left[\underset{ \mathcal{A}(\vec{S})}{\mathbb{E}}\left[\hat{\mathcal{R}}_{S_{i_{\mathcal{A}(\vec{S})}}}(h_{\mathcal{A}(\vec{S})})\right]\right]+\underset{\vec{S} \sim \mathcal{D}^{kN}}{\mathbb{E}}\left[\underset{ \mathcal{A}(\vec{S})}{\mathbb{E}}\left[\hat{\mathcal{R}}_{S_{i_{\mathcal{A}(\vec{S})}}}(h_{\mathcal{A}(\vec{S})})\right]\right]
		\\
		\leq& 1-e^{-\varepsilon}+e^{-\varepsilon}k\delta.
		\end{align*}
		
		The other side of the inequality can be similarly obtained.
		
		The proof is completed.
	\end{proof}

	\subsection{Proofs of Theorem \ref{thm:high_probability_privacy}}
	\label{proof:generalization}
	
	%We need two lemmas to derive Theorem \ref{thm:high_probability_privacy}. 
	We then prove Theorem \ref{lem:high_multi} and Lemma \ref{lemma:counterexample} to prove Theorem \ref{thm:high_probability_privacy}.
	The proofs are inspired by \cite{nissim2015generalization} but we have made significant development. %We also give their proofs for the integrity of the paper.
	
	\begin{proof}[Proof of Theorem \ref{lem:high_multi}]
		By Corollary \ref{coro:stab_multi}, we have that
		\begin{equation*}
		\underset{\vec{S} \sim \mathcal{D}^{kN}}{\mathbb{E}}\left[\underset{ \mathcal{A}(\vec{S})}{\mathbb{E}}\left[\hat{\mathcal{R}}_{S_{i_{\mathcal{A}(\vec{S})}}}(h_{\mathcal{A}(\vec{S})})\right]\right]\le e^{-\varepsilon}k \delta+\varepsilon+\underset{\vec{S} \sim \mathcal{D}^{kN}}{\mathbb{E}}\left[\underset{ \mathcal{A}(\vec{S})}{\mathbb{E}}\left[\mathcal{R}\left(h_{\mathcal{A}(\vec{S})}\right)\right]\right].
		\end{equation*}
		
		Since $\hat{\mathcal{R}}_{S_{i_{\mathcal{A}(\vec{S})}}} \left(h_{\mathcal{A}(\vec{S})}\right)\ge 0$,  we have that for any $\alpha> 0$, 
		\begin{align*}
		&\underset{\vec{S} \sim \mathcal{D}^{kN}}{\mathbb{E}}\left[\underset{ \mathcal{A}(\vec{S})}{\mathbb{E}}\left[\hat{\mathcal{R}}_{S_{i_{\mathcal{A}(\vec{S})}}}(h_{\mathcal{A}(\vec{S})})\right]\right]
		\\
		\ge& \underset{\vec{S} \sim \mathcal{D}^{kN}}{\mathbb{E}}\left[\underset{ \mathcal{A}(\vec{S})}{\mathbb{E}}\left[\hat{\mathcal{R}}_{S_{i_{\mathcal{A}(\vec{S})}}}(h_{\mathcal{A}(\vec{S})})\mathbb{I}_{\hat{\mathcal{R}}_{S_{i_{\mathcal{A}(\vec{S})}}}(h_{\mathcal{A}(\vec{S})})\ge \mathcal{R}\left(h_{\mathcal{A}(\vec{S})}\right)+\alpha}\right]\right]
		\\
		\ge& \underset{\vec{S} \sim \mathcal{D}^{kN}}{\mathbb{E}}\left[\underset{ \mathcal{A}(\vec{S})}{\mathbb{E}}\left[(\alpha+\mathcal{R}\left(h_{\mathcal{A}(\vec{S})}\right))\mathbb{I}_{\hat{\mathcal{R}}_{S_{i_{\mathcal{A}(\vec{S})}}}(h_{\mathcal{A}(\vec{S})})\ge \mathcal{R}\left(h_{\mathcal{A}(\vec{S})}\right)+\alpha}\right]\right].
		\end{align*}
		
		Furthermore, by splitting  $\underset{\vec{S} \sim \mathcal{D}^{kN}}{\mathbb{E}}[\underset{ \mathcal{A}(\vec{S})}{\mathbb{E}}[\mathcal{R}\left(h_{\mathcal{A}(\vec{S})}\right)]]$ into two parts, we have 
		\begin{align*}
		&\underset{\vec{S} \sim \mathcal{D}^{kN}}{\mathbb{E}}\left[\underset{ \mathcal{A}(\vec{S})}{\mathbb{E}}\left[\hat{\mathcal{R}}_{S_{i_{\mathcal{A}(\vec{S})}}}(h_{\mathcal{A}(\vec{S})})\right]\right]-\underset{\vec{S} \sim \mathcal{D}^{kN}}{\mathbb{E}}\left[\underset{ \mathcal{A}(\vec{S})}{\mathbb{E}}\left[\mathcal{R}\left(h_{\mathcal{A}(\vec{S})}\right)\right]\right] 
		\\
		\ge &\underset{\vec{S} \sim \mathcal{D}^{kN}}{\mathbb{E}}\left[\underset{ \mathcal{A}(\vec{S})}{\mathbb{E}}\left[(\alpha+\mathcal{R}\left(h_{\mathcal{A}(\vec{S})}\right))\mathbb{I}_{\hat{\mathcal{R}}_{S_{i_{\mathcal{A}(\vec{S})}}}(h_{\mathcal{A}(\vec{S})})\ge \mathcal{R}\left(h_{\mathcal{A}(\vec{S})}\right)+\alpha}\right]\right]
		\\
		&-\left(\underset{\vec{S} \sim \mathcal{D}^{kN}}{\mathbb{E}}\left[\underset{ \mathcal{A}(\vec{S})}{\mathbb{E}}\left[\mathcal{R}\left(h_{\mathcal{A}(\vec{S})}\right)\mathbb{I}_{\hat{\mathcal{R}}_{S_{i_{\mathcal{A}(\vec{S})}}}(h_{\mathcal{A}(\vec{S})})\ge \mathcal{R}\left(h_{\mathcal{A}(\vec{S})}\right)+\alpha}\right]\right]\right.
		\\&+\left.\underset{\vec{S} \sim \mathcal{D}^{kN}}{\mathbb{E}}\left[\underset{ \mathcal{A}(\vec{S})}{\mathbb{E}}\left[\mathcal{R}\left(h_{\mathcal{A}(\vec{S})}\right)\mathbb{I}_{\hat{\mathcal{R}}_{S_{i_{\mathcal{A}(\vec{S})}}}(h_{\mathcal{A}(\vec{S})})< \mathcal{R}\left(h_{\mathcal{A}(\vec{S})}\right)+\alpha}\right]\right]\right)
		\\
		\ge& \underset{\vec{S} \sim \mathcal{D}^{kN}}{\mathbb{E}}\left[\underset{ \mathcal{A}(\vec{S})}{\mathbb{E}}\left[\alpha\mathbb{I}_{\hat{\mathcal{R}}_{S_{i_{\mathcal{A}(\vec{S})}}}(h_{\mathcal{A}(\vec{S})})\ge \mathcal{R}\left(h_{\mathcal{A}(\vec{S})}\right)+\alpha}\right]\right]
		\\&-\left.\underset{\vec{S} \sim \mathcal{D}^{kN}}{\mathbb{E}}\left[\underset{ \mathcal{A}(\vec{S})}{\mathbb{E}}\left[\mathcal{R}\left(h_{\mathcal{A}(\vec{S})}\right)\mathbb{I}_{\hat{\mathcal{R}}_{S_{i_{\mathcal{A}(\vec{S})}}}(h_{\mathcal{A}(\vec{S})})< \mathcal{R}\left(h_{\mathcal{A}(\vec{S})}\right)+\alpha}\right]\right]\right)
		\\
		\ge& \alpha\mathbb{P}\left(\hat{\mathcal{R}}_{S_{i_{\mathcal{A}(\vec{S})}}}(h_{\mathcal{A}(\vec{S})})> \mathcal{R}\left(h_{\mathcal{A}(\vec{S})}\right)+\alpha\right)-\mathbb{P}\left(\hat{\mathcal{R}}_{S_{i_{\mathcal{A}(\vec{S})}}}(h_{\mathcal{A}(\vec{S})})\le \mathcal{R}\left(h_{\mathcal{A}(\vec{S})}\right)+\alpha\right).
		\end{align*}
		Let $\alpha=e^{-\varepsilon}k\delta+3\varepsilon$. We have
		\begin{equation*}
		\mathbb{P}\left(\hat{\mathcal{R}}_{S_{i_{\mathcal{A}(\vec{S})}}}(h_{\mathcal{A}(\vec{S})})\le \mathcal{R}\left(h_{\mathcal{A}(\vec{S})}\right)+\alpha\right)\le \frac{\alpha-(e^{-\varepsilon}k\delta+\varepsilon)}{1+\alpha}\le \varepsilon.
		\end{equation*}
		The proof is completed.
	\end{proof}
	
	We then prove Lemma \ref{lemma:counterexample}.
	
	\begin{proof}[Proof of Lemma \ref{lemma:counterexample}]
		Construct algorithm $\mathcal{B}$ with input $\vec{S}=\{S_i\}_{i=1}^k$ and $T$ ( where $S_i, T \in \mathcal{Z}^N$) as follows: 
		
		\textbf{Step 1.} Run $\mathcal{A}$ on $S_i$, $i=1,\cdots,k$. Denote the output as $h_i=\mathcal{A}(S_i)$.
		
		\textbf{Step 2.} Let utility function as $q(\vec{S},T,i)=N\left(\hat{\mathcal{R}}_{S_i}(h_i)-\hat{\mathcal{R}}_T(h_i)\right)$. Apply the utility $q$ to an $\varepsilon$-differential private exponential mechanism $\mathcal{M}(h_i,\vec{S},T)$ and return the output.
		
		We then prove that $\mathcal{B}$ satisfies
		\begin{equation*}
		\mathbb{P}\left[l\left(h_{\mathcal{B}(\vec{S})},S_{i_{\mathcal{B}(\vec{S})}}\right) \leq \mathcal{R}\left(h_{\mathcal{B}(\vec{S})}\right)+k e^{-\varepsilon} \delta+ 3\varepsilon\right] < \varepsilon.
		\end{equation*}
		
		By eq. (\ref{eq:condition_gene}), we have that 
		\begin{equation*}
		\mathbb{P}\left(\forall i, \hat{\mathcal R}(\mathcal{A}(S_i)) \leq e^{-\varepsilon}k \delta+ 8\epsilon+{\mathcal R}(\mathcal{A}(S_i))\right)\le \left(1-\frac{e^{-\varepsilon}\delta}{\varepsilon} \ln \left(\frac{2}{\epsilon}\right)\right)^k,
		\end{equation*}
		which leads to 
		\begin{equation}
		\label{eq:contro_1}
		\mathbb{P}\left(\exists i, \hat{\mathcal R}(\mathcal{A}(S_i)) > e^{-\varepsilon}k \delta+ 8\epsilon+{\mathcal R}(\mathcal{A}(S_i))\right)> 1-\left(1-\frac{1}{k} \ln \left(\frac{2}{\epsilon}\right)\right)^k\ge 1-\frac{\varepsilon}{2}.
		\end{equation}
		Furthermore, since $T$ is independent with $\vec{S}$, by Hoeffding inequality, we have that 
		\begin{equation}
		\label{eq:contro_2}
		\mathbb{P}\left(\forall i, \vert l(h_i,T)- \mathcal{R}(h_i)\vert\le \frac{\varepsilon}{2}\right)\ge( 1-e^{- \epsilon^{2} / 2N})^k\ge 1-\frac{\varepsilon}{8}.
		\end{equation}
		
		Therefore, combining eq. (\ref{eq:contro_1}) and eq. (\ref{eq:contro_2}),
		\begin{equation*}
		\mathbb{P}\left(\exists i, \hat{\mathcal R}(\mathcal{A}(S_i)) > e^{-\varepsilon}k \delta+ \frac{15}{2}\epsilon+l(h_i,T)\right)> 1-\frac{5\varepsilon}{8}.
		\end{equation*} 
		
		Since $q$ has senstivity $1$, we have that fixed $h_i$
		\begin{equation*}
		\mathbb{P}\left(\mathcal{M}(h_i,\vec{S},T)\le \text{OPT}(q(\vec{S},T,i))-N\varepsilon)\right)\ge 1-\frac{\varepsilon}{4},
		\end{equation*} 
		which leads to 
		\begin{equation*}
		\mathbb{P}\left( \hat{\mathcal R}_{S_{i_{\mathcal{B}(\vec{S})}}}(h_{\mathcal{B}(\vec{S})}) > e^{-\varepsilon}k \delta+ \frac{13}{2}\epsilon+\hat{\mathcal R}_{T}(h_{\mathcal{B}(\vec{S})})\right)> 1-\frac{7\varepsilon}{8}.
		\end{equation*} 
		
		Then, using eq. (\ref{eq:contro_2}) again, we have
		\begin{equation*}
		\mathbb{P}\left( \hat{\mathcal R}_{S_{i_{\mathcal{B}(\vec{S})}}}(h_{\mathcal{B}(\vec{S})}) > e^{-\varepsilon}k \delta+ 6\epsilon+\mathcal{R}(h_{\mathcal{B}(\vec{S})})\right)> 1-\varepsilon.
		\end{equation*} 
	\end{proof}

\section{Proofs of Composition Theorems}
\label{proof:composition}

%\subsection{Proofs of Composition Theorems}
{
This section proves the composition theorems. It is organized as follows: %Section \ref{sec:composition_preliminaries} recalls some notions and preliminaries; 
Section \ref{subsubsec:proof_of_lemma_kl} proves a preparation lemma on the KL divergence $D_{KL}(\mathcal{A}(S)\Vert \mathcal{A}(S'))$ between the hypotheses $\mathcal{A}(S)$ and $\mathcal{A}(S')$; based on this lemma Section \ref{subsubsec:pure} proves a composition theorem of $\varepsilon$-differential privacy; Section \ref{subsubsec:general} extends the composition theorem to $(\varepsilon, \delta)$-differential privacy;  Section \ref{sec:moment} further tightens the estimate of $\delta'$ under some assumptions; and Section \ref{sec:tightness_moment} analyses the tightness of this estimation.
}

\subsection{Proof of Lemma \ref{lemma: relation of kl}}
\label{subsubsec:proof_of_lemma_kl}
	
\begin{proof}[Proof of Lemma \ref{lemma: relation of kl}]
	By Lemma \ref{lemma: antipodal}, we have a random variable $M(S)$ and $M(S')$, which satisfies 
	\begin{equation*}
	D_{\infty}(M(S)\Vert M(S'))\le \varepsilon,\text{ }	D_{\infty}(M(S')\Vert M(S))\le \varepsilon,
	\end{equation*}
	and 
	\begin{equation}
	\label{eq:kl_dual}
	D_{KL}(\mathcal{A}(S)\Vert\mathcal{A}(S'))\le D_{KL}(M(S)\Vert M(S'))=D_{KL}(M(S')\Vert M(S)).
	\end{equation}
	
	Therefore, we only need to derive a bound for $D_{KL}(M(S)\Vert M(S'))$.
	
	By direct calculation,
	\begin{align}
	&\nonumber D_{KL}(M(S)\Vert M(S'))
	\\\nonumber
	\overset{(*)}{=}& \frac{1}{2}\left[D_{KL}(M(S)\Vert M(S'))+D_{KL}(M(S')\Vert M(S))\right]
	\\
	\nonumber=& \frac{1}{2}\int \log \frac{\text{d} \mathbb P(M(S))}{\text{d} \mathbb P(M(S'))} \text{d}\mathbb P(M(S))+ \frac{1}{2}\int \log \frac{\text{d} \mathbb P(M(S'))}{\text{d} \mathbb P(M(S))} \text{d}\mathbb P(M(S'))
	\\
	\nonumber=& \frac{1}{2}\int \log \frac{\text{d} \mathbb P(M(S))}{\text{d} \mathbb P(M(S'))} \text{d}\left[\mathbb P(M(S))-\mathbb P(M(S'))\right]
	\\
	\nonumber&+ \frac{1}{2}\int \left(\log \frac{\text{d} \mathbb P(M(S'))}{\text{d} \mathbb P(M(S))}
	+\log \frac{\text{d} \mathbb P(M(S))}{\text{d} \mathbb P(M(S'))}\right) \text{d}\mathbb P(M(S'))
	\\\nonumber
	=& \frac{1}{2}\int \log \frac{\text{d} \mathbb P(M(S))}{\text{d} \mathbb P(M(S'))} \text{d}\left[\mathbb P(M(S))-\mathbb P(M(S'))\right]+ \frac{1}{2}\int \log 1 \text{ }  \text{d}\mathbb P(M(S'))
	\\
	=& \frac{1}{2}\int \log \frac{\text{d} \mathbb P(M(S))}{\text{d} \mathbb P(M(S'))} \text{d}\left[\mathbb P(M(S))-\mathbb P(M(S'))\right],	\label{eq: kl middle outcome}
	\end{align}
	{where eq. ($*$) comes from eq. (\ref{eq:kl_dual}).}
	
	We now analyse the last integration in eq. (\ref{eq: kl middle outcome}). Define 
	\begin{equation}
	\label{eq:def_ky}
	    k(y)\overset{\triangle}{=}\frac{\text{d} \mathbb P(M(S)=y)}{\text{d} \mathbb P(M(S')=y)}-1.
	\end{equation}
	
	Therefore,
	\begin{equation}
	\label{eq:constrain k}
	k(y)\text{d} \mathbb P(M(S')=y) = \text d \mathbb P(M(S)=y)-\text{d} \mathbb P(M(S')=y).
	\end{equation}
	
	Additionally,
	\begin{align}
	\label{eq:ky_expectation}
	    \mathbb E_{M(S')} k(M(S') = & \int_{y \in \mathcal H} k(y) \text{d}\mathbb P(M(S')=y) \nonumber\\
	    = & \int_{y \in \mathcal H} \text d \left( \mathbb P(M(S)=y)-\text{d} \mathbb P(M(S')=y) \right) \nonumber\\
	    = & 0.
	\end{align}
	
	By calculating the integration of the both sides of eq. (\ref{eq:constrain k}), we have
	\begin{equation*}
	\label{eq: int constain k}
	\int k(y)\text{d} \mathbb P(M(S')=y)=0.
	\end{equation*}
	
	Also, combined with the definition of $k(y)$ (see eq. \ref{eq:def_ky}), the right-hand side (RHS) of eq. (\ref{eq: kl middle outcome}) becomes 
	\begin{equation}
	\label{eq: kl expectation form}
	\text{RHS} = \mathbb E_{M(S')} k(M(S')) \log(k(M(S'))+1).
	\end{equation}
	
	Since $M$ is $\varepsilon$-differentially private, $k(y)$ is bounded from both sides as follows,
	\begin{equation}
	\label{eq:k range}
	e^{-\varepsilon}-1\le k(y) \le e^{\varepsilon} -1.
	\end{equation}

	We now calculate the maximum of eq. (\ref{eq: kl expectation form}) subject to  eqs. (\ref{eq:ky_expectation}) and (\ref{eq:k range}).

%	\begin{align}
%	\label{eq: condition 1}
%	e^{-\varepsilon}-1\le k(y) \le e^{\varepsilon} -1,\\
%	\label{eq: condition 2}
%	\mathbb E_{M(S')}  k(M(S')) =0.
%	\end{align}
	
	First, we argue that the maximum is achieved when $k(M(S')) \in \{e^{-\varepsilon}-1,e^{\varepsilon}-1\}$ with probability $1$ (almost surely). When $k(M(S')) \in \{e^{-\varepsilon}-1,e^{\varepsilon}-1\}$, almost surely, %combined with the condition eq. \ref{eq:ky_expectation}, 
	the distribution for $k(M(S'))$ is as following,
	\begin{gather*}
	\mathbb P^{*}(k(M(S'))=e^{\varepsilon}-1)=\frac{1}{1+e^{\varepsilon}},\\
	\mathbb P^{*}(k(M(S'))=e^{-\varepsilon}-1)=\frac{e^{\varepsilon}}{1+e^{\varepsilon}}.
	\end{gather*}
	We argue that it is the distribution that maximizes $k(M(S'))$.
	
	For the brevity, we denote the probability measure for a given distribution $Q$ as $\mathbb P_{Q}$. Similarly, $\mathbb P^*$ corresponds the distribution $Q^*$. We prove that $Q^{*}$ maximizes eq. (\ref{eq: kl expectation form}) in the following two cases: (1) $\mathbb P_{Q}(k(M(S'))\ge 0)\le \mathbb P^{*}(k(M(S'))=e^{\varepsilon}-1)$, and (2) $\mathbb P_{Q}(k(M(S'))\ge 0)>$ $ \mathbb P^{*}(k(M(S'))=e^{\varepsilon}-1)$
	
	\textbf{Case 1:} $\mathbb P_{Q}(k(M(S'))\ge 0)\le \mathbb P^{*}(k(M(S'))=e^{\varepsilon}-1)$ 
	
	We have 
	\begin{align*}
	&\mathbb{E}_{M(S')\sim Q^{*}}(k(M(S')\log (k(M(S'))+1))
	\\
	=&\mathbb P^{*}(k(M(S'))= e^{\varepsilon}-1)\cdot\varepsilon (e^{\varepsilon}-1)- \mathbb P^{*}(k(M(S'))=e^{-\varepsilon}-1)\cdot\varepsilon(e^{-\varepsilon}-1)
	\\
	=&(\mathbb P^{*}(k(M(S'))=e^{\varepsilon}-1)-\mathbb P_{Q}(k(M(S'))\ge 0))\cdot \varepsilon(e^{\varepsilon}-1)
	\\
	&+\mathbb P_{Q}(k(M(S'))\ge 0)\cdot \varepsilon(e^{\varepsilon}-1)
	- \mathbb P^{*}(k(M(S'))=e^{-\varepsilon}-1) \cdot \varepsilon(e^{-\varepsilon}-1)
	\\
	\ge& \mathbb P_{Q}(k(M(S'))\ge 0)\cdot \varepsilon(1-e^{-\varepsilon})
	- \mathbb P^{*}(k(M(S'))=e^{-\varepsilon}-1) \cdot \varepsilon(1-e^{-\varepsilon})
	\\
	&+\mathbb P_{Q}(k(M(S'))\ge 0)\cdot \varepsilon(e^{\varepsilon}-1)
	- \mathbb P^{*}(k(M(S'))=e^{-\varepsilon}-1) \cdot \varepsilon(e^{-\varepsilon}-1).
	\end{align*}
	
	Note that 
	\begin{align*}
	\mathbb P_{Q}(k(M(S'))< 0)=&\mathbb P^{*}(k(M(S'))=e^{\varepsilon}-1)-\mathbb P_{Q}(k(M(S'))\ge 0)\\
	&+ \mathbb P^{*}(k(M(S'))=e^{-\varepsilon}-1).
	\end{align*}
	
	Therefore, together with the condition eq. (\ref{eq:k range}),
	\begin{align}
	\nonumber
	&\mathbb{E}_{M(S')\sim Q}(k(M(S')\log (k(M(S'))+1)I_{k(M(S')\le 0)})
	\\
	\nonumber
	\le&(\mathbb P^{*}(k(M(S'))=e^{\varepsilon}-1)-\mathbb P_{Q}(k(M(S'))\ge 0))\cdot \varepsilon(1-e^{-\varepsilon})\\
	\label{eq:case1-}
	&+\mathbb P^{*}(k(M(S'))=e^{-\varepsilon}-1) \cdot \varepsilon(1-e^{-\varepsilon}).
	\end{align}
	
	Also,
	\begin{equation}
	\label{eq:case1+}
	\mathbb{E}_{M(S')\sim Q}(k(M(S')\log (k(M(S'))+1)I_{k(M(S'))>0}) \le \mathbb P_{Q}(k(M(S'))\ge 0)\cdot \varepsilon(e^{\varepsilon}-1).
	\end{equation}
	
	Therefore, combined inequalities eqs. (\ref{eq:case1-}) and (\ref{eq:case1+}), we have
	\begin{equation*}
	\label{eq:case1}
	\mathbb{E}_{M(S')\sim Q}(k(M(S')\log (k(M(S'))+1)) \le \mathbb{E}_{M(S')\sim Q^{*}}(k(M(S')\log (k(M(S'))+1)).
	\end{equation*}
	
	Since the distribution $Q$ is arbitrary, the distribution $Q^*$ maximizes the $k(M(S')\log (k(M(S'))+1)$.
	
	\textbf{Case 2:} $\mathbb P_{Q}(k(M(S'))\ge 0)> \mathbb P^{*}(k(M(S'))=e^{\varepsilon}-1)$
	
	We first prove that if $\mathbb P_Q(1-e^{-\varepsilon}<k(M(S'))<0)\ne 0$, there exists a distribution $Q'$ such that
\begin{gather*}
    \mathbb P_{Q'}(k(M(S'))\ge 0)=\mathbb P_{Q}(k(M(S'))\ge 0),\\
    \mathbb P_{Q'}(k(M(S'))< 0)=\mathbb P_{Q}(k(M(S'))< 0),\\
    \mathbb P_{Q'}(k(M(S'))< 0)=\mathbb P_{Q'}(k(M(S')=e^{-\varepsilon}-1),\\
    \mathbb E_{Q'}(k(M(S')\log (k(M(S'))+1))>\mathbb E_{Q'}(k(M(S')\log (k(M(S'))+1)),
\end{gather*}
while the two conditions (eqs. \ref{eq:ky_expectation}, \ref{eq:k range}) still hold.

Additionally, we have assumed that
\begin{equation*}
    \mathbb P_{Q}(k(M(S'))\ge 0)> \mathbb P^{*}(k(M(S'))=e^{\varepsilon}-1).
\end{equation*}

Therefore,
\begin{equation*}
    \mathbb P_{Q}(k(M(S'))\le 0)<\mathbb P^{*}(k(M(S'))=e^{-\varepsilon}-1).
\end{equation*}
	
Also, since the distribution $Q'$ is arbitrary, let it satisfy 
	\begin{equation*}
	\mathbb P_{Q'}(k(M(S'))< 0)=\mathbb P_{Q}(k(M(S'))< 0)=\mathbb P_{Q'}(k(M(S')=e^{-\varepsilon}-1).
	\end{equation*}
	
	Then, in order to meet the condition eq. (\ref{eq:ky_expectation}), let
	\begin{equation*}
	    \mathbb P_{Q'}(k(M(S')=e^{\varepsilon}-1) > \mathbb P_{Q}(k(M(S')=e^{\varepsilon}-1),
	\end{equation*}
	 and
	 \begin{equation*}
	     \mathbb P_{Q'}(0<k(M(S'))<e^{\varepsilon}-1) \le \mathbb P_{Q}(0<k(M(S'))<e^{\varepsilon}-1),
	 \end{equation*}

Since $x\log(x+1)$ increases when $x>0$ and decreases when $x<0$, we have
\begin{equation*}
    \mathbb E_{Q'}(k(M(S')\log (k(M(S'))+1))>\mathbb E_{Q}(k(M(S')\log (k(M(S'))+1)).
\end{equation*}
	
	Therefore, we have proved that the argument when $\mathbb P_{Q}(k(M(S'))< 0) \ne \mathbb P_{Q}(k(M(S'))=e^{-\varepsilon}-1)$. We now prove the case that
	\begin{equation*}
	   \label{eq:case_2}
	   \mathbb P_{Q}(k(M(S'))< 0)=\mathbb P_{Q}(k(M(S'))=e^{-\varepsilon}-1),
	\end{equation*}
where
	\begin{equation*}
	\mathbb E_{Q}(k(M(S')\log (k(M(S'))+1)I_{k(M(S'))<0})=\varepsilon(1-e^{-\varepsilon})\mathbb P_{Q}(k(M(S'))< 0).
	\end{equation*}
	
	Applying Jensen's inequality to bound the $\mathbb E_{Q}(k(M(S')\log (k(M(S'))+1)I_{k(M(S'))\ge 0})$, we have
	\begin{align}
	\label{eq:ky_log}
	&\mathbb E_{Q}(k(M(S'))\log (k(M(S'))+1)I_{k(M(S'))\ge 0}) \nonumber\\
	=&\mathbb P_{Q}(M(S')\ge 0)  \mathbb E_{Q'}(k(M(S'))\log (k(M(S'))+1)|k(M(S'))\ge 0) \nonumber
	\\
	\overset{(*)}{\le} &\mathbb P_{Q}(M(S')\ge 0) \mathbb E_Q\left(k(M(S'))|k(M(S'))\ge 0\right) \cdot \log( \mathbb E_Q\left(k(M(S'))|k(M(S'))\ge 0\right)+1),
	\end{align}
	where the inequality ($*$) uses Jensen's inequality ($x\log(1+x)$ is convex with respect to $x$ when $x>0$). The upper bound in eq. (\ref{eq:ky_log}) is achieved as long as
	\begin{equation*}
	    \mathbb P_{Q}(k(M(S'))\ge 0)=\mathbb P_{Q}(k(M(S'))=\mathbb E_Q(k(M(S'))|k(M(S'))\ge 0)).
	\end{equation*}
	
	Furthermore,
	\begin{equation*}
	    \mathbb P_{Q}(k(M(S'))< 0)=\mathbb P_{Q}(k(M(S'))=e^{-\varepsilon}-1).
	\end{equation*}
	
	Therefore, the distribution $Q$ is determined by the cumulative density functions $\mathbb P_{Q}(k(M(S'))< 0)$ and $\mathbb P_{Q}(k(M(S'))\ge 0)$.
	
	Hence, maximizing $ \mathbb E_{Q}(k(M(S')\log (k(M(S'))+1))$ is equivalent to maximizing the following object function,
	\begin{equation*}
	\label{eq: final function}
	g(q)=q(1-e^{-\varepsilon}) \log e^{\varepsilon}+(1-q) \frac{q}{1-q}(1-e^{-\varepsilon}) \log\left(\frac{q}{1-q}(1-e^{-\varepsilon})+1\right),
	\end{equation*}
	subject to 
	\begin{equation}
	\label{eq: final condition}
	\frac{q}{1-q}\le e^{\varepsilon},
	\end{equation}
	where $g(q)$ is the maximum of eq. (\ref{eq: kl expectation form}) subject to $\mathbb P_{Q}(k(M(S'))< 0)=q$, and the condition eq. (\ref{eq: final condition}) comes from the $\mathbb P_{Q}(k(M(S'))\ge 0)> \mathbb P^{*}(k(M(S'))=e^{\varepsilon}-1)$ (the assumption of Case 2).
	
	Additionally, $g(q)$ can be represented as follows,
	\begin{equation*}
	q(1-e^{-\varepsilon}) \log\left(\frac{q}{1-q}(e^{\varepsilon}-1)+\varepsilon\right).
	\end{equation*}
	
	Since both $q$ and $\frac{q}{1-q}$ monotonously increase, $g(q)$ monotonously increases. Therefore, $Q^{*}$ maximize eq. (\ref{eq: kl expectation form}), which finishes the proof.
\end{proof}

\subsection{Proof of Theorem \ref{thm:composition_IV}}

\label{subsubsec:pure}
Based on Lemma \ref{lemma: relation of kl}, we can prove the following composition theorem for $\varepsilon$-differential privacy as a preparation theorem of the general case.

%\begin{theorem}[Composition Theorem of $\varepsilon$-Differential Privacy]
%\label{thm:composition_V}
%	Suppose an iterative machine learning algorithm $\mathcal A$ has $T$ steps: $\left\{Y_i(S)\right\}_{i=1}^T$. Specifically, $M_i: (Y_{i-1}(S), S) \mapsto Y_{i}(S)$ is the $i$-th iterator. Assume that $Y_0$ is the initial hypothesis (which does not depend on $S$). If for any fixed observation $y_{i-1}$ of the variable $Y_{i-1}$, $M_i(y_{i-1},S)$ is $\varepsilon_i$-differentially private, 
%	then $\left\{Y_i(S)\right\}_{i=0}^T$ is ($\varepsilon'$, $\delta'$)-differentially private that
%	\begin{equation*}
%	\varepsilon^{\prime}=\sqrt{2  \log \left( \frac{1}{\delta'}\right)\left(\sum\limits_{i=1}^T \varepsilon_{i}^2\right)} +\sum\limits_{i=1}^T \varepsilon_i \frac{e^{\varepsilon_i}-1}{e^{\varepsilon_i}+1}.
%	\end{equation*}
%\end{theorem}

\begin{proof}[Proof of Theorem \ref{thm:composition_IV}]
	We begin by calculating $\log \frac{\mathbb P\left(\left\{Y_i(S)=y_i\right\}_{i=0}^T\right)}{\mathbb P\left(\left\{Y_i(S')=y_i\right\}_{i=0}^T\right)}$ as follows,
	\begin{align*}
	&\log \frac{\mathbb P\left(\left\{Y_i(S)=y_i\right\}_{i=0}^T\right)}{\mathbb P\left(\left\{Y_i(S')=y_i\right\}_{i=0}^T\right)} \\
	=& \log \left( \prod_{i = 0}^T \frac{\mathbb P\left(Y_i(S)=y_i \vert Y_{i-1}(S)=y_{i-1},...,Y_{0}(S)=y_{0}\right)}{\mathbb P\left(Y_i(S')=y_i\vert Y_{i-1}(S')=y_{i-1},...,Y_{0}(S')=y_{0} \right)} \right)\\
	=&\sum\limits_{i=0}^{T} \log\left(\frac{\mathbb P\left(Y_i(S)=y_i \vert Y_{i-1}(S)=y_{i-1},...,Y_{0}(S)=y_{0}\right)}{\mathbb P\left(Y_i(S')=y_i\vert Y_{i-1}(S')=y_{i-1},...,Y_{0}(S')=y_{0} \right)}\right)\\
	\overset{(*)}{=}&\sum\limits_{i=1}^{T} \log\left(\frac{\mathbb P\left(Y_i(S)=y_i \vert Y_{i-1}(S)=y_{i-1},...,Y_{0}(S)=y_{0}\right)}{\mathbb P\left(Y_i(S')=y_i\vert Y_{i-1}(S')=y_{i-1},...,Y_{0}(S')=y_{0} \right)}\right)\\
	=&\sum\limits_{i=1}^{T} \log\left(\frac{\mathbb P\left(M_i(y_{i-1},S)=y_i \vert Y_{i-1}(S)=y_{i-1},...,Y_{0}(S)=y_{0}\right)}{\mathbb P\left(M_i(y_{i-1},S')=y_i\vert Y_{i-1}(S')=y_{i-1},...,Y_{0}(S')=y_{0} \right)}\right)\\
	\overset{(**)}{=}&\sum\limits_{i=1}^{T}\log\left(\frac{\mathbb P\left(M_i(y_{i-1},S)=y_i \right)}{\mathbb P\left(M_i(y_{i-1},S')=y_i \right)}\right),
	\end{align*}
	where eq. ($*$) comes from the independence of $Y_0$ with respect to $S$ and eq. ($**$) is because the independence of $M_i$ to $Y_k$ ($k < i$) when the $Y_{i-1}$ is fixed.
	
	By the definition of $\varepsilon$-differential privacy, one has for arbitrary $y_{i-1}$ as the observation of $Y_{i-1}$,
	\begin{gather*}
	D_\infty\left(M_i(y_{i-1},S) \| M_i(y_{i-1},S')\right) <\varepsilon_i,\\
	D_\infty\left(M_i(y_{i-1},S')\|M_i(y_{i-1},S)\right) <\varepsilon_i.
	\end{gather*}
	
	Thus, by Lemma \ref{lemma: relation of kl}, we have that
	\begin{align}
	\nonumber& \mathbb E\left(\log\left(\frac{\mathbb P\left(M_i(Y_{i-1},S)=Y_i \right)}{\mathbb P\left(M_i(Y_{i-1},S')=Y_i \right)}\right)\Big| Y_{i-1}(S)=y_{i-1},...,Y_{0}(S)=y_{0} \right)\\
	=&\nonumber D_{KL}(M_i(y_{i-1},S)\| M_i(y_{i-1},S'))\\
	\label{eq:bound for kl}
	\le& \varepsilon_i \frac{e^{\varepsilon_i}-1}{e^{\varepsilon_i}+1}.
	\end{align}
	
	Combining Azuma Lemma (Lemma \ref{lemma:Azuma}), eq. (\ref{eq:bound for kl}) derives the following equation
	\begin{equation*}
	\label{eq:excess prob}
	\mathbb P\left(\left\{Y_i(S)=y_i\right\}_{i=0}^T: \frac{\mathbb P\left(\left\{Y_i(S)=y_i\right\}_{i=0}^T\right)}{\mathbb P\left(\left\{Y_i(S')=y_i\right\}_{i=0}^T\right)}>e^{\varepsilon'}\right)<\delta',
\end{equation*}
where $S$ and $S'$ are neighbour sample sets.

Therefore, the algorithm $\mathcal A$ is $\varepsilon'$-differentially private.

The proof is completed.
	
		%We start by bounding the probability that $\left(Y_i\right)_{i=0}^T$ is not  $\varepsilon'$-differentially private, i.e., we prove this conclusion by showing that for every two neighbor databases (or samples) S and S',
%	\begin{equation}
%	\label{eq:excess prob}
%	\mathbb P\left(\left(Y_i(S)=y_i\right)_{i=0}^T: \frac{\mathbb P\left(\left(Y_i(S)=y_i\right)_{i=0}^T\right)}{\mathbb P\left(\left(Y_i(S')=y_i\right)_{i=0}^T\right)}>e^{\varepsilon'}\right)<\delta'.
%	\end{equation}
\end{proof}

\subsection{Proof of Theorem \ref{thm:composition_V}}
\label{subsubsec:general}

Now, we can prove our composition theorems for $(\varepsilon, \delta)$-differential privacy. We first prove a composition algorithm of $(\varepsilon, \delta)$-differential privacy whose estimate of $\varepsilon'$ is somewhat looser than the existing results. Then, we tighten the results and obtain a composition theorem that strictly tighter than the current estimate.

%\begin{theorem}[Composition Theorem III]
%	\label{thm:composition_V}
%		Suppose an iterative machine learning algorithm $\mathcal A$ has $T$ steps: $\left\{Y_i(S)\right\}_{i=1}^T$. Specifically, $M_i: (Y_{i-1}(S), S) \mapsto Y_{i}(S)$ is the $i$-th iterator. Assume that $Y_0$ is the initial hypothesis (which does not depend on $S$). If for any fixed observation $y_{i-1}$ of the variable $Y_{i-1}$, $M_i(y_{i-1},S)$ is ($\varepsilon_i$, $\delta_i$)-differentially private ($i\ge 1$),
%	then $\left\{Y_i(S)\right\}_{i=0}^T$ is ($\varepsilon'$, $\delta'$) differentially private with
%	\begin{gather*}
%	\varepsilon^{\prime}=\sqrt{2  \log \left( \frac{1}{\tilde\delta}\right)\left(\sum\limits_{i=1}^T \varepsilon_{i}^2\right)} +\sum\limits_{i=1}^T \varepsilon_i \frac{e^{\varepsilon_i}-1}{e^{\varepsilon_i}+1},\\
%	\delta'=\max_{\{\alpha_i\}_{i=1}^T \in I} 1-\prod_{i=1}^T\left(1-e^{\alpha_i}\frac{\delta_i}{1+e^{\varepsilon_i}}\right)+1-\prod_{i=1}^T\left(1-\frac{\delta_i}{1+e^{\varepsilon_i}}\right)+\tilde\delta,
%	\end{gather*}
%	where $I=\left\{\{\alpha_i\}_{i=1}^T: \sum_{i=1}^T \alpha_i=\varepsilon' \text{ and } |\{i:\alpha_i\ne \varepsilon_i \text{ and }\alpha_i\ne 0\}|\le 1\right\}$.
%\end{theorem}

	\begin{proof}[Proof of Theorem \ref{thm:composition_V}]
	    
	    It has been proved that the optimal privacy preservation can be achieved by a sequence of independent iterations (see \cite{kairouz2017composition},  Theorem 3.5). Therefore, without loss of generality, we assume that the iterations in our theorem are independent with each other.
		
 		Rewrite $Y_i(S)$ as $Y_i^0$, and $Y_i(S')$ as $Y_i^1$ ($i\ge 1$). Then, by Lemma \ref{lemma:bridge}, there exist random variables $\tilde{Y}_i^0$ and $\tilde{Y}_i^1$, such that 
		\begin{align}
		\label{eq:Y_and_Y'}
		\Delta\left(Y_i^0\Vert\tilde{Y}_i^0\right) \le & \frac{\delta_i}{1+e^{\varepsilon_i}}, \\
		\label{eq:Z_and_Z'}
		\Delta\left(Y_i^1\Vert\tilde{Y}_i^1\right) \le & \frac{\delta_i}{1+e^{\varepsilon_i}},\\
		\label{eq:Y'_and_Z'_1}
		D_{\infty} \left(\tilde{Y}_i^0\Vert\tilde{Y}_i^1\right) \le & \varepsilon_i,\\
		\label{eq:Y'_and_Z'_2}
		D_{\infty} \left(\tilde{Y}_i^1\Vert\tilde{Y}_i^0\right) \le & \varepsilon_i.
		\end{align}
		
		Applying Theorem \ref{thm:composition_V} (here, $\delta=\tilde{\delta}$), we have that 
		\begin{gather*}
		D_{\infty}^{\tilde{\delta}}\left(\{\tilde{Y}^0_i\}_{i=0}^T\Vert\{\tilde{Y}_i^1\}_{i=0}^T\right)\le \varepsilon',\\
		D_{\infty}^{\tilde{\delta}}\left(\{\tilde{Y}^1_i\}_{i=0}^T\Vert\{\tilde{Y}_i^0\}_{i=0}^T\right)\le \varepsilon'.
		\end{gather*}
		
%		It is possible that
%		\begin{equation*}
%		\mathbb P(Y_i^0 \in B_0)-\frac{\delta_i}{1+e^{\varepsilon_i}} < 0.
%		\end{equation*}
%		In this situation, we can replace $\frac{\delta_i}{1+e^{\varepsilon_i}}$ by $\mathbb P(Y_i^0 \in B_0)$, in order to make the inequality stands. This modification actually can be viewed as lowering $\frac{\delta_i}{1+e^{\varepsilon_i}}$ and thus, lowering the total bound for $\delta'$. Therefore, without loss of generality we assume that
        Apparently,
        \begin{equation*}
	    \mathbb P(Y_i^0 \in B_i)-\min\left\{ \frac{\delta_i}{1+e^{\varepsilon_i}},\mathbb P(Y_i^0 \in B_i)\right\} \ge 0.
		\end{equation*}
		
		Therefore, for any sequence of hypothesis sets $B_0$, $\cdots$, $B_T$,
		\begin{align}
		\nonumber
		&\mathbb P(Y_0^0 \in B_0)\left(\mathbb P(Y_1^0 \in B_1)-\min\left\{ \frac{\delta_1}{1+e^{\varepsilon_1}},\mathbb P(Y_1^0 \in B_1)\right\}\right) \nonumber\\
		& \cdots\left(\mathbb P(Y_T^0 \in B_T)-\min\left\{\frac{\delta_T}{1+e^{\varepsilon_T}},\mathbb P(Y_T^0 \in B_1)\right\}\right) \nonumber\\
		\label{eq:Y^0_composition}
		\le &\mathbb P(\tilde{Y}^0_0\in B_0)\cdots \mathbb P(\tilde{Y}^0_T\in B_T)\nonumber\\
		\le & e^{\varepsilon'}\mathbb P(\tilde{Y}^1_0\in B_0)\cdots \mathbb P(\tilde{Y}^1_T\in B_T)+\tilde{\delta}.
		\end{align}
		
		%Here, we acknowledge that 

		%Notice that 
		Furthermore, by eq. (\ref{eq:Y'_and_Z'_2}), we also have that 
		\begin{equation*}
		\mathbb P(\tilde{Y}^0_i \in B_i) \le \min\left\{e^{\varepsilon_i},\frac{1}{\mathbb P(\tilde{Y}^1_i \in B_i)}\right\}\mathbb P(\tilde{Y}^1_i \in B_i).
		\end{equation*}
		Therefore,
		\begin{equation*}
		\mathbb P(\tilde{Y}^0_0\in B_0)\cdots \mathbb P(\tilde{Y}^0_n\in B_T)
		\le
		\prod_{i=1}^T\min\left\{e^{\varepsilon_i},\frac{1}{\mathbb P(\tilde{Y}^1_i \in B_i)}\right\}\mathbb P(\tilde{Y}^1_0\in B_0)\cdots \mathbb P(\tilde{Y}^1_T\in B_T)+\tilde\delta.
		\end{equation*}
		Then, we prove this theorem in two cases:
		(1) $\mathbf{\prod_{i=1}^T\min\left\{e^{\varepsilon_i},\frac{1}{\mathbb P(\tilde{Y}^1_i \in B_i)}\right\} \le e^{\varepsilon'}}$; and
		(2) $\\\mathbf{\prod_{i=1}^T\min\left\{e^{\varepsilon_i},\frac{1}{\mathbb P(\tilde{Y}^1_i \in B_i)}\right\} > e^{\varepsilon'}}$.
		
		\textbf{Case 1}- $\mathbf{\prod_{i=1}^T\min\left\{e^{\varepsilon_i},\frac{1}{\mathbb P(\tilde{Y}^1_i \in B_i)}\right\} \le e^{\varepsilon'}}$.
		
		We have that 
		\begin{align*}
		&\mathbb P(\tilde{Y}_0^1\in B_0)\left(\mathbb P(\tilde{Y}_1^1 \in B_1)-\frac{\delta_1}{1+e^{\varepsilon_1}}\right)\cdots \left(\mathbb P(\tilde{Y}_T^1 \in B_T)-\frac{\delta_T}{1+e^{\varepsilon_T}}\right)
		\\
		\le& \mathbb P({Y}_0^1 \in B_0)\cdots \mathbb P({Y}_T^1 \in B_T).
		\end{align*}
		
		By simple calculation, we have that 
		\begin{align*}
		&\prod_{i=1}^T\min\left\{e^{\varepsilon_i},\frac{1}{\mathbb P(\tilde{Y}^1_i \in B_i)}\right\}\mathbb P(\tilde{Y}_0^1 \in B_0)\cdots \mathbb P(\tilde{Y}_T^1 \in B_T)
		\\
		\le& \prod_{i=1}^T\min\left\{e^{\varepsilon_i},\frac{1}{\mathbb P(\tilde{Y}^1_i \in B_i)}\right\}\mathbb P({Y}_0^1 \in B_0)\cdots \mathbb P({Y}_T^1 \in B_T)
		\\
		&+
		\prod_{i=1}^T\min\left\{e^{\varepsilon_i},\frac{1}{\mathbb P(\tilde{Y}^1_i \in B_i)}\right\}\mathbb P(\tilde{Y}_0^1 \in B_0)\cdots \mathbb P(\tilde{Y}_T^1 \in B_T)
		\\
		&-\prod_{i=1}^n\min\left\{e^{\varepsilon_i},\frac{1}{\mathbb P(\tilde{Y}^1_i \in B_i)}\right\}\mathbb P(\tilde{Y}_0^1\in B_0)\\
		&~~~~~~~~~~\left(\mathbb P(\tilde{Y}_1^1 \in B_1)
		-\frac{\delta_1}{1+e^{\varepsilon_1}}\right)\cdots \left(\mathbb P(\tilde{Y}_T^1 \in B_T)-\frac{\delta_T}{1+e^{\varepsilon_T}}\right).
		\end{align*}
		
		Apparently, 
		\begin{equation*}
		\min\left\{e^{\varepsilon_i},\frac{1}{\mathbb P(\tilde{Y}^1_i \in B_i)}\right\} \mathbb P(\tilde{Y}_0^i \in B_i) \le 1,
		\end{equation*}
		and when $A>B$, $f(x)=Ax-(x-a)B$ increases when $x$ increases.
		
		Therefore, we have that
		\begin{align*}
		&\prod_{i=1}^T\min\left\{e^{\varepsilon_i},\frac{1}{\mathbb P(\tilde{Y}^1_i \in B_i)}\right\}\mathbb P(\tilde{Y}_0^1 \in B_0)\cdots \mathbb P(\tilde{Y}_T^1 \in B_T)
		\\
		-&\prod_{i=1}^T\min\left\{e^{\varepsilon_i},\frac{1}{\mathbb P(\tilde{Y}^1_i \in B_i)}\right\}P(\tilde{Y}_0^1 \in B_0)\\
		&~~~~~~\left(\mathbb P(\tilde{Y}_1^1 \in B_1)
		-\frac{\delta_1}{1+e^{\varepsilon_1}}\right)\cdots \left(\mathbb P(\tilde{Y}_T^1 \in B_T)-\frac{\delta_T}{1+e^{\varepsilon_T}}\right)
		\\
		\le& 1-\prod_{i=1}^T\left(1-\min\left\{e^{\varepsilon_i},\frac{1}{\mathbb P(\tilde{Y}^1_i \in B_i)}\right\}\frac{\delta_i}{1+e^{\varepsilon_i}}\right).
		\end{align*}
		
		Combining with eq. (\ref{eq:Y^0_composition}), we have that 
		\begin{equation*}
		\delta'\le 1-\prod_{i=1}^T\left(1-\min\left\{e^{\varepsilon_i},\frac{1}{\mathbb P(\tilde{Y}^1_i \in B_i)}\right\}\frac{\delta_i}{1+e^{\varepsilon_i}}\right)+1-\prod_{i=1}^T\left(1-\frac{\delta_i}{1+e^{\varepsilon_i}}\right)+\tilde{\delta}.
		\end{equation*}
		
		\textbf{Case 2}- $\mathbf{\prod_{i=1}^T\min\left\{e^{\varepsilon_i},\frac{1}{\mathbb P(\tilde{Y}^1_i \in B_i)}\right\} > e^{\varepsilon'}}$:
		
		There exists a sequence of reals $\{\alpha_i\}_{i=1}^T$ such that 
		\begin{gather*}
		e^{\alpha_i} \le \min\left\{e^{\varepsilon_i},\frac{1}{\mathbb P(\tilde{Y}^1_i \in B_i)}\right\},\\
		\sum_{i=1}^T\alpha_i =\varepsilon'.
		\end{gather*}
		
		Therefore, similar to Case 1, we have that
		\begin{equation*}
		\delta'\le 1-\prod_{i=1}^T\left(1-e^{\alpha_i}\frac{\delta_i}{1+e^{\varepsilon_i}}\right)+1-\prod_{i=1}^T\left(1-\frac{\delta_i}{1+e^{\varepsilon_i}}\right).
		\end{equation*}
		
		Overall, we have proven that
		\begin{equation*}
		\label{eq:optimization_composition}
		\delta'\le 1-\prod_{i=1}^T\left(1-e^{\alpha_i}\frac{\delta_i}{1+e^{\varepsilon_i}}\right)+1-\prod_{i=1}^T\left(1-\frac{\delta_i}{1+e^{\varepsilon_i}}\right),
		\end{equation*}
		where $\sum_{i=1}^T \alpha_i \le \varepsilon'$ and $\alpha_i \le \varepsilon_i$.
		
		From Lemma \ref{lemma:opt}, the minimum is realised on the boundary, which is exactly this theorem claims.
		
		The proof is completed.
		%So,
%		Therefore, we only need to optimize the RHS in eq. (\ref{eq:optimization_composition}) under the following constraints:
%		\begin{equation*}
%		\begin{cases}
%		\sum_{i=1}^T \alpha_i = \varepsilon',
%		\\
%		\alpha_i \le \varepsilon_i.
%		\end{cases}
%		\end{equation*}
%	    By Lemma \ref{lemma:opt}, we finish the proof.
	\end{proof}

%{\color{red}Then, we use the technique of Theorem \ref{coro:composition} to derive a composition theorem when the $(\varepsilon, \delta)$-differenrial privacy can vary among the iterators. It is the heterogeneous case of Theorem \ref{thm:composition_2}.}

Then, we can prove prove Theorem \ref{thm:composition_2}.

%{\color{red}
%\begin{theorem}[Composition Theorem III]
%	\label{thm:composition_of_hetero_case}
%	Suppose an iterative machine learning algorithm $\mathcal A$ has $T$ steps: $\left\{Y_i(S)\right\}_{i=0}^T$. Specifically, $M_i: (Y_{i-1}(S), S) \mapsto Y_{i}(S)$ is the $i$-th iterator. Assume that $Y_0$ is the initial hypothesis (which does not depend on $S$). If for any fixed observation $y_{i-1}$ of the variable $Y_{i-1}$,, $M_i(y_{i-1},S)$ is ($\varepsilon_i$, $\delta_i$)-differentially private ($i\ge 1$),
%	then $\left(Y_i(S)\right)_{i=0}^T$ is ($\varepsilon'$, $\delta'$) differentially private with
%	\begin{gather*}
%	\varepsilon'=\min \left\{\sum_{i=1}^{T} \varepsilon_{i}, \sum_{i=1}^{T} \frac{\left(e^{\varepsilon_{i}}-1\right) \varepsilon_{i}}{e^{\varepsilon_{i}}+1}+\sqrt{\sum_{i=1}^{T} 2 \varepsilon_{i}^{2} \log \left(e+\frac{\sqrt{\sum_{i=1}^{T} \varepsilon_{i}^{2}}}{\tilde{\delta}}\right)}, \sum_{i=1}^{T} \frac{\left(e^{\varepsilon_{i}}-1\right) \varepsilon_{i}}{e^{\varepsilon_{i}+1}}+\sqrt{\sum_{i=1}^{T} 2 \varepsilon_{i}^{2} \log \left(\frac{1}{\tilde{\delta}}\right)}\right\},\\
%	\delta'=\max_{\left\{\alpha_i\right\}_{i=1}^T \in I} 1-\prod_{i=1}^T\left(1-e^{\alpha_i}\frac{\delta_i}{1+e^{\varepsilon_i}}\right)+1-\prod_{i=1}^T\left(1-\frac{\delta_i}{1+e^{\varepsilon_i}}\right)+\tilde\delta,
%	\end{gather*}
%	here $I=\left\{\{\alpha_i\}_{i=1}^T:|\{i:\alpha_i\ne \varepsilon_i \text{ and }\alpha_i\ne 0\}|\le 1\right\}$.
%\end{theorem}
%}

\begin{proof}[Proof of Theorem \ref{thm:composition_2}]
     Applying Theorem 3.5 in \cite{kairouz2017composition} and replacing $\varepsilon'$ in the proof of Theorem {\ref{thm:composition_V}} as
      \begin{gather*}
      \varepsilon'=\min \left\{I_1, I_2, I_3\right\},
      \end{gather*}
      where
      \begin{gather*}
      I_1 = \sum_{i=1}^{T} \varepsilon_{i},\\
      I_2 = \sum_{i=1}^{T} \frac{\left(e^{\varepsilon_{i}}-1\right) \varepsilon_{i}}{e^{\varepsilon_{i}}+1}+\sqrt{\sum_{i=1}^{T} 2 \varepsilon_{i}^{2} \log \left(e+\frac{\sqrt{\sum_{i=1}^{T} \varepsilon_{i}^{2}}}{\tilde{\delta}}\right)},\\
      I_3 = \sum_{i=1}^{T} \frac{\left(e^{\varepsilon_{i}}-1\right) \varepsilon_{i}}{e^{\varepsilon_{i}}+1}+\sqrt{\sum_{i=1}^{T} 2 \varepsilon_{i}^{2} \log \left(\frac{1}{\tilde{\delta}}\right)}
      \end{gather*}
      
      The proof  is completed.
\end{proof}

\subsection{Proof of Corollary \ref{thm:composition_moment}}
\label{sec:moment}
%{\color{red}\begin{theorem}
%Suppose an iterative machine learning algorithm $\mathcal A$ has $T$ steps: $\left\{Y_i(S)\right\}_{i=0}^T$. Specifically, $M_i: (Y_{i-1}(S), S) \mapsto Y_{i}(S)$ is the $i$-th iterator. Assume that $Y_0$ is the initial hypothesis (which does not depend on $S$). If for any fixed observation $y_{i-1}$ of the variable $Y_{i-1}$,, $M_i(y_{i-1},S)$ is ($\varepsilon$, $\delta$)-differentially private ($i\ge 1$),
%then $\left(Y_i(S)\right)_{i=0}^T$ is ($\varepsilon'$, $\delta'$) differentially private with
%\begin{gather*}
%\varepsilon'=T\frac{e^\varepsilon-1}{e^\varepsilon+1}+\varepsilon\sqrt{2T\log\left(\frac{1}{\tilde\delta}\right)},
%\\
%\delta'=1-\left(1-e^{\varepsilon}\frac{\delta}{1+e^{\varepsilon}}\right)^{\left \lceil  \frac{\varepsilon'}{\varepsilon}\right \rceil}\left(1-\frac{\delta}{1+e^{\varepsilon}}\right)^{T-\left \lceil  \frac{\varepsilon'}{\varepsilon}\right \rceil}+1-\left(1-\frac{\delta}{1+e^{\varepsilon}}\right)^{T}+{\color{blue}\delta''},
%\end{gather*}
%where ${\color{blue} \delta''}$ is defned  as following:
%\begin{equation*}
% \delta'' =e^{-\frac{\varepsilon'+T\varepsilon}{2}}\left(\frac{1}{1+e^\varepsilon}\left(\frac{2T\varepsilon}{T\varepsilon-\varepsilon'}\right)\right)^T\left(\frac{T\varepsilon+\varepsilon'}{T\varepsilon-\varepsilon'}\right)^{-\frac{\varepsilon'+T\varepsilon}{2{\color{blue} \varepsilon}}}.
%\end{equation*}
%\end{theorem}}

\begin{proof}[Proof of Corollary \ref{thm:composition_moment}]
	Let $\mathcal P_0$ and $\mathcal P_1$ be two distributions whose cumulative distribution functions $P_0$ and $P_1$ are respectively defined as following:
	\begin{equation*}
	P_{0}(x) =\left\{\begin{aligned} &\delta  \text { , } &x=0 \\ &\frac{(1-\delta) e^{\varepsilon}}{1+e^{\varepsilon}}  \text { , } &x=1 \\ &\frac{1-\delta}{1+e^{\varepsilon}}  \text { , } &x=2 \\ &0  \text { , } &x=3 \end{aligned}\right.,
	\end{equation*} 
	and 
    \begin{equation*}
	P_{1}(x)= \left\{\begin{aligned} &0  \text { , } &x=0 \\ &\frac{(1-\delta) e^{\varepsilon}}{1+e^{\varepsilon}}  \text { , } &x=1 \\ &\frac{1-\delta}{1+e^{\varepsilon}}  \text { , }& x=2 \\ &\delta  \text { , } &x=3 \end{aligned}\right..
	\end{equation*}
		
	By Theorem 3.4 of \cite{kairouz2017composition}, the largest magnitude of the $(\varepsilon', \delta')$-differential privacy can be calculated from the $\mathcal P_0^T$ and $\mathcal P_1^T$.
	
	Construct $\tilde{\mathcal{P}}_0$ and $\tilde{\mathcal P}_1$, whose cumulative distribution functions are as follows,
	\begin{equation*}
	\tilde{P}_{0}(x) =\left\{\begin{aligned} &\frac{e^\varepsilon\delta}{1+e^{\varepsilon}}  \text { , }& x=0 
	\\ &\frac{(1-\delta) e^{\varepsilon}}{1+e^{\varepsilon}} \text { , }& x=1 \\ &\frac{1-\delta}{1+e^{\varepsilon}}  \text { , } &x=2 \\ &\frac{\delta}{1+e^{\varepsilon}}  \text { , }& x=3 \end{aligned}\right.,
	\end{equation*} 
	and 
	\begin{equation*}
	\tilde{ P}_{1}(x) \left\{\begin{aligned} &\frac{\delta}{1+e^{\varepsilon}} \text { , }& x=0 \\ &\frac{(1-\delta) e^{\varepsilon}}{1+e^{\varepsilon}}  \text { , }& x=1 \\ &\frac{1-\delta}{1+e^{\varepsilon}}  \text { , }& x=2 \\ &\frac{e^\varepsilon\delta}{1+e^{\varepsilon}}  \text { , }& x=3 \end{aligned}\right..
	\end{equation*}
	
	One can easily verify that
	\begin{gather*}
	\label{eq:P_0_and_P_0'}
	\Delta(\mathcal{P}_0\Vert\tilde{\mathcal{P}}_0)\le \frac{\delta}{1+e^{\varepsilon}},\\
	\label{eq:P_1_and_P_1'}
	\Delta(\mathcal{P}_1\Vert\tilde{\mathcal{P}}_1)\le \frac{\delta}{1+e^{\varepsilon}},\\
	\label{eq:P_0_and_P_1}
	D_{\infty} (\tilde{\mathcal{P}}_0\Vert\tilde{\mathcal{P}}_1)\le \varepsilon,\\
	D_{\infty} (\tilde{\mathcal{P}}_1\Vert\tilde{\mathcal{P}}_0)\le \varepsilon_i.
	\end{gather*}
	
	Let $V_i(x_i) = \log \left(\frac{\tilde{\mathcal{P}}_0(x_i)}{\tilde {\mathcal{P}}_1(x_i)}\right)$ and $S(x_1,\cdots,x_T)=\sum_{i=1}^T V_i(x_i)$.
	
	We have that for any $t>0$,
	\begin{align}
	\nonumber
	\mathbb P_{\tilde{ \mathcal P}_0^T}(\left\{x_i\right\}:S(\left\{x_i\right\})>\varepsilon')&\le e^{-\varepsilon't} \mathbb{E}_{\tilde{\mathcal P}_0^T} (e^{tS})
	\\\nonumber
	&=e^{-\varepsilon't}\left(\frac{e^{t\varepsilon+\varepsilon}}{1+e^{\varepsilon}}+\frac{e^{-t\varepsilon}}{1+e^{\varepsilon}}\right)^T
	\\
	\label{eq:probability_bound_moment}
	&=e^{-\varepsilon't-Tt\varepsilon}\left(\frac{e^{2t\varepsilon+\varepsilon}}{1+e^{\varepsilon}}+\frac{1}{1+e^{\varepsilon}}\right)^T.
	\end{align}
	
	By calculating the derivative,, we have that the minimum of the RHS of eq. (\ref{eq:probability_bound_moment}) is achieved at
	\begin{equation}
	\label{eq:t_minimum}
	e^{2\varepsilon t}=e^{-\varepsilon} \frac{T\varepsilon+\varepsilon'}{T\varepsilon-\varepsilon'}.
	\end{equation}
    
    Since $\varepsilon'\ge T\frac{e^\varepsilon-1}{e^\varepsilon+1}$,
    \begin{equation*}
        e^{-\varepsilon} \frac{T\varepsilon+\varepsilon'}{T\varepsilon-\varepsilon'}>1.
    \end{equation*}
    
	Therefore, by applying eq.(\ref{eq:t_minimum}) into the RHS of eq. (\ref{eq:probability_bound_moment}), we have that  
	\begin{align}
	\label{eq:moment_delta}
	\mathbb P_{\tilde{\mathcal P}_0^T}(\left\{x_i\right\}:S(\left\{x_i\right\})>\varepsilon')&\le e^{-\frac{\varepsilon'+T\varepsilon}{2}}\left(\frac{1}{1+e^\varepsilon}\left(\frac{2T\varepsilon}{T\varepsilon-\varepsilon'}\right)\right)^T\left(\frac{T\varepsilon+\varepsilon'}{T\varepsilon-\varepsilon'}\right)^{-\frac{\varepsilon'+T\varepsilon}{2{\varepsilon}}}.
	\end{align}
	
	Define RHS of eq. (\ref{eq:moment_delta}) as $\delta'$. We have $(\tilde{\mathbb  P}^b)^T$ ($b=0,1$) is ($\varepsilon'$,$\delta'$)-differentially private. Then, using similar analysis of the Proof of Theorem {\ref{thm:composition_V}}, we prove this theorem.
\end{proof}

\subsection{Tightness of Theorem \ref{thm:composition_moment}}
\label{sec:tightness_moment}
This section analyses the tightness of Theorem \ref{thm:composition_moment}. Specifically, we compare it with our Theorem \ref{thm:composition_2}. 

In the proof of Theorem \ref{thm:composition_2} (see Section \ref{subsubsec:pure}), $\varepsilon_3'$ is derived through Azuma Lemma (Lemma \ref{lemma:Azuma}). Specifically, the $\delta'$ is derived by 
\begin{align*}
\mathbb P \left[S_{T} \geq \varepsilon'-T\frac{e^\varepsilon-1}{e^\varepsilon+1}\right] & \leq e^{-t (\varepsilon'-T\frac{e^\varepsilon-1}{e^\varepsilon+1})} \mathbb{E}\left[e^{t S_{T}}\right]
 \\
 &=e^{-t (\epsilon'-T\frac{e^\varepsilon-1}{e^\varepsilon+1})} \mathbb{E}\left[e^{t S_{T-1}} \mathbb{E}\left[e^{t \mathbf{V}_{T}} | Y_{1}(S), \ldots, Y_{T-1}(S)\right]\right]
 \\
 & \leq e^{-t(\epsilon'-T\frac{e^\varepsilon-1}{e^\varepsilon+1})} \mathbb{E}\left[e^{t S_{T-1}}\right] e^{4t^{2} \varepsilon^{2} / 8}
 \\ & \leq e^{-t (\epsilon'-T\frac{e^\varepsilon-1}{e^\varepsilon+1})} e^{Tt^{2}\varepsilon^2 / 2},
 \end{align*}
 where $V_i$ is defined as $ \log\frac{\mathbb{P}(Y_i(S))}{\mathbb{P}(Y_i(S'))}-\mathbb E\left[\log\frac{\mathbb{P}(Y_i(S))}{\mathbb{P}(Y_i(S'))} | Y_{1}(S), \ldots, Y_{i-1}(S)\right]$ and $S_j$ is defined as $\sum_{i=1}^j V_i$.
 
 Since $\mathbb P \left[S_{T} \geq \varepsilon'-T\frac{e^\varepsilon-1}{e^\varepsilon+1}\right]$ does not depend on $t$,
 \begin{align*}
 \mathbb P \left[S_{T} \geq \varepsilon'-T\frac{e^\varepsilon-1}{e^\varepsilon+1}\right] \le \min_{t>0} e^{- \frac{(\epsilon'-T\frac{e^\varepsilon-1}{e^\varepsilon+1})^2}{2T\varepsilon^2} } = \delta',
\end{align*}

By contrary, the approach here directly calculates $\mathbb E [e^{tS_T}]$, without the shrinkage in the proof of Theorem \ref{thm:composition_2} (see Section \ref{subsubsec:pure}). Specifically, 
 \begin{align*}
 e^{-\varepsilon't-Tt\varepsilon}\left(\frac{e^{2t\varepsilon+\varepsilon}}{1+e^{\varepsilon}}+\frac{1}{1+e^{\varepsilon}}\right)^T=e^{-t \epsilon} \mathbb{E}\left[e^{t S_{T}}\right]%\\
 \le e^{-t (\epsilon'-T\frac{e^\varepsilon-1}{e^\varepsilon+1})} e^{Tt^{2}\varepsilon^2 / 2}.
 \end{align*}
 Therefore, 
 \begin{equation*}
 \min_{t>0} e^{-\varepsilon't-Tt\varepsilon}\left(\frac{e^{2t\varepsilon+\varepsilon}}{1+e^{\varepsilon}}+\frac{1}{1+e^{\varepsilon}}\right)^T \le   \min_{t>0}e^{-t (\epsilon'-T\frac{e^\varepsilon-1}{e^\varepsilon+1})} e^{Tt^{2}\varepsilon^2 / 2},
 \end{equation*}
 which leads to 
 \begin{equation*}
  e^{-\frac{\varepsilon'+T\varepsilon}{2}}\left(\frac{1}{1+e^\varepsilon}\left(\frac{2T\varepsilon}{T\varepsilon-\varepsilon'}\right)\right)^T\left(\frac{T\varepsilon+\varepsilon'}{T\varepsilon-\varepsilon'}\right)^{-\frac{\varepsilon'+T\varepsilon}{2{\varepsilon}}}\le  \delta'.
 \end{equation*}
 It ensures that this estimate further tightens $\delta'$ than Section \ref{subsubsec:pure} (which is also the $\tilde \delta$ in Theorem \ref{thm:composition_2}) if the $\varepsilon'$ is the same.

 \section{Applications}
 \label{app:application}
 
This appendix collects the proofs for the applications in SGLD and federated learning.
 
\subsection{Proof of Theorem \ref{thm:SGLD}}

\begin{proof}[Proof of Theorem \ref{thm:SGLD}]
	
	We first calculate the differential privacy of each step. Assume mini-batch $\mathcal{B}$ has been selected and define $\nabla\hat{\mathcal{R}}_S^{\tau}(\theta)$ as following:
	\begin{equation*}
	\nabla\hat{\mathcal{R}}_S^{\tau}(\theta)=\nabla r(\theta)+ \sum_{z \in \mathcal{B}}\nabla l(z | \theta).
	\end{equation*}
	 For any two neighboring sample sets $S$ and $S'$ and fixed $\theta_{i-1}$, we have %the differential privacy of the $i$-th iteration is
	\begin{align*}
	\max_{\theta_i}\frac{p(\theta^S_{i}=\theta_i\vert \theta^S_{i-1}=\theta_{i-1})}{p(\theta^{S'}_{i}=\theta_i\vert \theta^{S'}_{i-1}=\theta_{i-1})}
	&=\max_{\theta_i}\frac{p(\eta_i(-\frac{1}{\tau}\nabla\hat{\mathcal{R}}_S^{\tau}(\theta_{i-1})+\mathcal{N}(0, \sigma^2 \mathbf{I}))=\theta_i-\theta_{i-1})}{p(\eta_i(-\frac{1}{\tau}\nabla\hat{\mathcal{R}}_{S'}^{\tau}(\theta_{i-1})+\mathcal{N}(0, \sigma^2 \mathbf{I}))=\theta_i-\theta_{i-1})}
	\\
	&=\max_{\theta_i'}\frac{p(\eta_i(-\frac{1}{\tau}\nabla\hat{\mathcal{R}}_S^{\tau}(\theta_{i-1})+\mathcal{N}(0, \sigma^2 \mathbf{I}))=\theta_i')}{p(\eta_i(-\frac{1}{\tau}\nabla\hat{\mathcal{R}}_{S'}^{\tau}(\theta_{i-1})+\mathcal{N}(0, \sigma^2 \mathbf{I}))=\theta_i')}.
	\end{align*}
	
	% by arranging and taking logarithm, define $D(\theta')$ as follows:}
	Define
	\begin{equation*}
	D(\theta')=\log \frac{p(-\frac{1}{\tau}\nabla\hat{\mathcal{R}}_S^{\tau}(\theta_{i-1})+\mathcal{N}(0, \sigma^2 \mathbf{I})=\theta')}{p(-\frac{1}{\tau}\nabla\hat{\mathcal{R}}_{S'}^{\tau}(\theta_{i-1})+\mathcal{N}(0, \sigma^2 \mathbf{I})=\theta')},
	\end{equation*}
	where $\theta'=\frac{1}{\eta_i} \theta_i'$ obeys $-\frac{1}{\tau}\nabla\hat{\mathcal{R}}^{\tau}_S(\theta_{i-1})+\mathcal{N}(0, \sigma^2 \mathbf{I})$.%, since $\theta_i'$ obeys $\eta_i(-\frac{1}{\tau}\nabla\hat{\mathcal{R}}_S^{\tau}(\theta_{i-1})+\mathcal{N}(0, \sigma^2 \mathbf{I}))$.
	
	Let $\theta''=\theta'+\frac{1}{\tau}\nabla\hat{\mathcal{R}}^{\tau}_S(\theta_{i-1})$ and rewrite $D(\theta')$ as:
	\begin{align*}
	D(\theta')=&\log \frac{e^{-\frac{\Vert\theta'+\frac{1}{\tau}\nabla\hat{\mathcal{R}}_{S}^{\tau}(\theta_{i-1}))\Vert^2}{2\sigma^2}}}{e^{-\frac{\Vert\theta'+\frac{1}{\tau}\nabla\hat{\mathcal{R}}_{S'}^{\tau}(\theta_{i-1})\Vert^2}{2\sigma^2}}}
	\\
	=&-\frac{\Vert\theta'+\frac{1}{\tau}\nabla\hat{\mathcal{R}}_{S}^{\tau}(\theta_{i-1}))\Vert^2}{2\sigma^2}+\frac{\Vert\theta'+\frac{1}{\tau}\nabla\hat{\mathcal{R}}_{S'}^{\tau}(\theta_{i-1})\Vert^2}{2\sigma^2}
	\\
	=&-\frac{\Vert\theta''\Vert^2}{2\sigma^2}+\frac{\Vert\theta''+\frac{1}{\tau}\nabla\hat{\mathcal{R}}_{S'}^{\tau}(\theta_{i-1})-\frac{1}{\tau}\nabla\hat{\mathcal{R}}_{S}^{\tau}(\theta_{i-1}))\Vert^2}{2\sigma^2}
	\\
	=&\frac{2\theta''^T\frac{1 }{\tau}(\nabla\hat{\mathcal{R}}_{S'}^{\tau}(\theta_{i-1})-\nabla\hat{\mathcal{R}}_{S}^{\tau}(\theta_{i-1}))+\frac{1}{\tau^2}(\Vert\nabla\hat{\mathcal{R}}_{S'}^{\tau}(\theta_{i-1})-\nabla\hat{\mathcal{R}}_{S}^{\tau}(\theta_{i-1})\Vert^2)}{2\sigma^2}.
	\end{align*}
	
	Define $\nabla\hat{\mathcal{R}}_{S'}^{\tau}(\theta_{i-1})-\nabla\hat{\mathcal{R}}_{S}^{\tau}(\theta_{i-1})$ as $\mathbf{v}$. By definition of $L$, we have that
	\begin{equation*}
	\Vert \mathbf{v} \Vert<2 L.
	\end{equation*}
	
	Therefore, since $\theta''^T\mathbf{v} \sim \mathcal{N}(0,\Vert \mathbf{v}\Vert^2\sigma^2)$,
	by Chernoff Bound technique,
	\begin{align*}
	\mathbb P\left(\theta''^T\mathbf{v}\ge 2\sqrt{2}L\sigma\sqrt{\log\frac{1}{\delta}}\right)& \le	\mathbb P\left(\theta''^T\mathbf{v}\ge \sqrt{2}\Vert \mathbf{v}\Vert\sigma\sqrt{\log\frac{1}{\delta}}\right)\\
	&\le \min_{t}e^{-\sqrt{2}t\Vert \mathbf{v}\Vert\sigma\sqrt{\log\frac{1}{\delta}}}\mathbb{E}(e^{t\theta''^T\mathbf{v}})
	\\
	&=\delta.
	\end{align*}

	Therefore, with probability at least $1-\delta$ with respect to $\theta'$, we have that 
	\begin{equation*}
	D(\theta') \le \frac{2\sqrt{2}L\sigma\frac{1}{\tau}\sqrt{\log\frac{1}{\delta}}+\frac{4}{\tau^2}L^2}{2\sigma^2}.
	\end{equation*}
	
	Define $\varepsilon=\frac{2\sqrt{2}L\sigma\frac{1}{\tau}\sqrt{\log\frac{1}{\delta}}+\frac{4}{\tau^2}L^2}{2\sigma^2}$.
	Applying Lemma 4.4 in \cite{beimel2010bounds}, we have that the iteration $-\frac{1}{\tau} \nabla \hat{ \mathcal{R}}_S^{\tau}(\theta_{i-1})+\mathcal{N}(0, \sigma^2 \mathbf{I})$ is ($2\frac{\tau}{N}\varepsilon$, $\frac{\tau}{N}\delta$)-differentially private. Applying Theorem \ref{thm:composition_moment} and
	\begin{equation*}
	    \varepsilon'=\sqrt{8  \log \left( \frac{1}{\tilde\delta}\right)\left(\frac{\tau^2}{N^2}T\varepsilon^2\right)} +2T \frac{\tau}{N} \varepsilon \frac{e^{2\frac{\tau}{N}\varepsilon}-1}{e^{2\frac{\tau}{N}\varepsilon}+1},
	\end{equation*} we can prove the differential privacy.

	{
		Letting $\mathcal{B}$ sampled randomly and applying Theorem \ref{thm:high_probability_privacy}, we can prove  the  generalization bound. %$e^{-\varepsilon'} M \delta' +(1-e^{-\varepsilon'})M$. 

	The proof is completed.
}
\end{proof}

\subsection{Proof of Theorem \ref{thm:federated}}

%\subsection{Proof of Theorem \ref{thm:federated}}

We only need to prove differential privacy part of Theorem \ref{thm:federated}, and the rest of the proof is similar with the one of Theorem \ref{thm:SGLD}.
\begin{proof}[Proof of Theorem \ref{thm:federated}]
	The proof bears resemblance to the proof of Theorem \ref{thm:SGLD}. One only has to notice that each update is still a Gauss mechanism, while 
	\begin{equation*}
	\left\Vert\frac{h^t_i}{\max(1,\frac{\Vert h^t_i\Vert_2}{L})}\right\Vert\le L.
	\end{equation*}
	
	Then, in this situation, $D(\theta')$ is as follows:
	\begin{equation*}
	D(\theta')=\log \frac{p\left(\frac{1}{\tau} \left(\sum\limits_{c_k \in \mathcal{B}} \frac{h_i^t}{\max(1,\frac{\Vert h_i^t\Vert_2}{L})}\right)+ \mathcal{N}(0,L^2\sigma^2 \mathbf{I})=\theta'\right)}{p\left(\frac{1}{\tau} \left(\sum\limits_{c_k \in \mathcal{B}} \frac{h_i^t}{\max(1,\frac{\Vert h_i^t\Vert_2}{L})}\right)+ \mathcal{N}(0,L^2\sigma^2 \mathbf{I})=\theta'\right)}.
	\end{equation*}
	
	All other reasoning is the same as the previous proof.
	
	By Theorem \ref{thm:composition_moment} and
	\begin{equation*}
	\varepsilon'=\sqrt{8  \log \left( \frac{1}{\tilde\delta}\right)\left(\frac{\tau^2}{N^2}T\varepsilon^2\right)} +2T \frac{\tau}{N} \varepsilon \frac{e^{2\frac{\tau}{N}\varepsilon}-1}{e^{2\frac{\tau}{N}\varepsilon}+1},
	\end{equation*}
	we can calculate the differential privacy of federated learning. 

	The proof is completed.
\end{proof}

\end{document}